\documentclass[letterpaper, 10 pt, conference]{ieeeconf}  

\usepackage{graphicx}
\usepackage{subcaption}
\usepackage{amsmath}
\usepackage{siunitx}
\usepackage{longtable,tabularx}
\usepackage{algorithm}
\usepackage{algpseudocode}
\usepackage{amssymb}
\usepackage{dirtytalk}
\usepackage{multirow,tabularx}
\usepackage{caption,psfrag}
\usepackage{amsfonts}
\newtheorem{theorem}{Theorem}[section]
\newtheorem{corollary}{Corollary}[theorem]
\newtheorem{lemma}[theorem]{Lemma}
\usepackage{hyperref}
\usepackage[normalem]{ulem}
\usepackage{gensymb}
\newcommand{\minimize}{\operatorname{minimize}}
\usepackage{xcolor}
\newcommand{\argmin}{\operatorname{argmin}}

\IEEEoverridecommandlockouts                              
\newtheorem{prop}{\textit{Proposition}}
\title{\LARGE \bf
3D-OGSE: Online Safe and Smooth Trajectory Generation using Generalized Shape Expansion in Unknown 3-D Environments
}

\author{Vrushabh Zinage$^{1}$, Senthil Hariharan Arul$^{2}$, Dinesh Manocha$^{2}$, Satadal Ghosh$^{1}$
\thanks{$^{1}$Vrushabh Zinage and Satadal Ghosh are with the Department of Aerospace Engineering, Indian Institute of Technology Madras, Chennai, India
        {\tt\small ae16b017@smail.iitm.ac.in, satadal@iitm.ac.in}}%
\thanks{$^{2}$Senthil Hariharan Arul and Dinesh Manocha are at University of Maryland, College Park, USA
        {\tt\small dm@cs.umd.edu}}
}

\begin{document}
\bibliographystyle{IEEEtran}

\maketitle
\thispagestyle{empty}
\pagestyle{empty}

\begin{abstract}
In this paper, we present an online motion planning algorithm (3D-OGSE) for generating smooth, collision-free trajectories over multiple planning iterations for 3-D agents  operating in an unknown obstacle-cluttered 3-D environment. Our approach constructs a safe-region, termed 'generalized shape', at each planning iteration, which represents the obstacle-free region based on locally-sensed environment information. A collision-free path is computed by sampling points in the generalized shape and is used to generate a smooth, time-parametrized trajectory by minimizing snap. The generated trajectories are constrained to lie within the generalized shape, which ensures the agent maneuvers in the locally obstacle-free space. As the agent reaches boundary of 'sensing shape' in a planning iteration, a re-plan is triggered by receding horizon planning mechanism that also enables initialization of the next planning iteration. Theoretical guarantee of probabilistic completeness over the entire environment and of completely collision-free trajectory generation is provided. We evaluate the proposed method in simulation on complex 3-D environments with varied obstacle-densities. We observe that each re-planing computation takes $\sim$1.4 milliseconds on a single thread of an Intel Core i5-8500 3.0 GHz CPU. In addition, our method is found to perform 4-10 times faster than several existing algorithms. In simulation over complex scenarios such as narrow passages also we observe less conservative behavior.
\end{abstract}
\section{INTRODUCTION}
Due to their small size and superior agility, UAVs (unmanned aerial vehicles) are increasingly being used for various applications including search and rescue, surveillance, and exploration. Generating smooth feasible collision-free trajectories for such 3-D agents, especially in obstacle-dense or cluttered 3-D environments, forms a key problem. 
This becomes more challenging when the operating environment is not known prior to the flight ~\cite{scaramuzza,Augugliaro} as it is not possible to pre-compute a collision-free trajectory in such scenarios.  
As a consequence, we need efficient online planning algorithms that can rapidly generate a smooth and safe trajectory for 3-D agents using only local sensor data.

There has been extensive research over last few decades on trajectory generation~\cite{Augugliaro,Chen,Kushleyev}.
Prior trajectory generation methods can broadly be grouped into three categories: (1) optimization-based algorithms~\cite{Kushleyev,Deits}; (2) search-based/motion primitives~\cite{MacAllister,Likhachev,Pivtoraiko}; and (3) sampling-based methods~\cite{sertac,lavalle,zinage}. Except few of these algorithms like \cite{Likhachev,Pivtoraiko}, most of them work based on \textit{a-priori} known complete information of environment, thus posing a limitation for real-time implementation in navigation of 3-D agent through unknown environments.
\begin{figure}[]
\centering
\begin{subfigure}[b]{0.23\textwidth}
{\centerline{\includegraphics[width=0.9\textwidth,height=3.4cm]{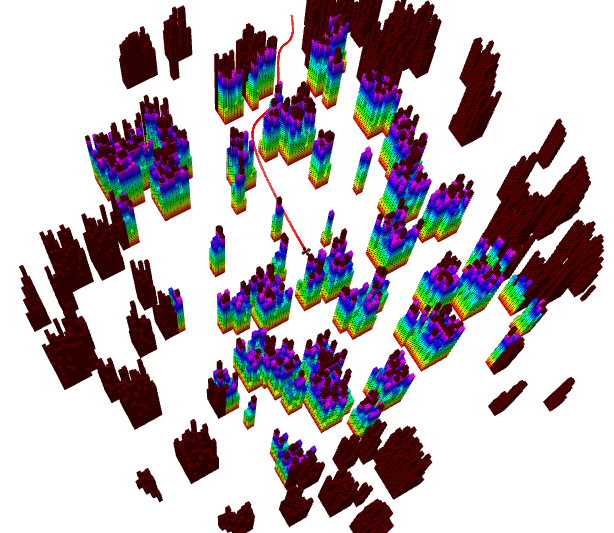}}}
\end{subfigure}
\begin{subfigure}[b]{0.23\textwidth}
\centerline{\includegraphics[width=0.9\textwidth,height=3.6cm]{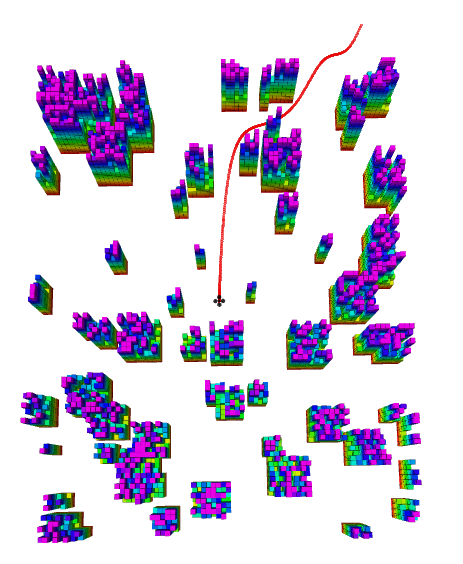}}
\end{subfigure}
\caption{A 3-D agent maneuvering in a random forest environment using 3D-OGSE. The agent has a sensing region of edge 10m, and the multicolored obstacles represent the obstacles inside the agent's sensing region. The brown obstacles are not visible to the agent at the current instance. The red curve represents an illustration of a planned trajectory generated by our approach, which can handle obstacle-dense 3-D environments.}
\label{fig:real_traj_1}
\vspace{-15pt}
\end{figure}
To this end, for online smooth and safe trajectory generation, optimization-based ~\cite{Augugliaro, bsplineusenko2017real} and search-based methods ~\cite{chen2016online, upenn, hkust} have been explored in recent literature. In the latter approach, the free-space is approximated as a sequence of convex shapes to reduce the complexity and enable fast computation, and a notion of safe corridor has been adopted for safe online trajectory generation. 
It is also important to design online trajectory generation methods that can provide guarantees of algorithm completeness.


\textbf{Main Contributions:} 
We present a 3D online GSE (named as '3D-OGSE') algorithm for generating collision-free smooth dynamically feasible trajectories for 3-D agents operating in obstacle-cluttered \textit{a-priori} unknown 3-D environments. Our approach is based on recent offline algorithms that use the notion of the generalized shape~\cite{zinage,zinage_3d_gse_journal}. We exploit the property that sampling-based algorithms are computationally more efficient than search-based methods for higher dimensions.  Based on the locally available information on environment at every planning iteration, a receding horizon planning (RHP) mechanism is framed for trajectory re-planning. While horizon of the RHP mechanism is restricted by the sensing region of on-board sensors of 3-D agent, the re-planning trigger is planned based on a novel notion of 'sensing shape' defined using sensing region and the notion of the generalized shape \cite{zinage_3d_gse_journal}. Once a feasible path is generated at every planning iteration, snap minimization quadratic problem is formulated for generating smooth optimal trajectories that are dynamically feasible~\cite{mellinger}. We show that the use of generalized shape in sampling points during path planning and in constraining trajectory points in trajectory optimization step at every planning iteration renders the generated trajectory completely collision-free. 
The {{novel}} components of our online approach include:
\begin{enumerate}
\item Fast and safe path planning leveraging a novel notion of `sensing shape' derived from on-board sensing region and 3D generalized shape at every planning iteration. 
\item A practically realizable sensing-restricted receding horizon planning mechanism and
fast trajectory optimization for generating dynamically feasible safe trajectory generation.
\item Theoretical guarantee of probabilistic completeness of the presented 3D-OGSE for generation of collision-free feasible path from start point to goal point over multiple planning iterations. Besides, we provide guarantee on the safety of generated overall trajectory.
\item Extensive evaluation of the efficacy of the presented 3D-OGSE algorithm under FOV-restricted, noisy sensing scenario in simulated complex environments. Extensive comparison study of the same with several other well-studied algorithms in literature reveals that the trajectory cost of the 3D-OGSE is comparable to 
them, while in terms of computation time, the 3D-OGSE is about $4-10\times$ faster than them.
\end{enumerate}

The outline of the paper is as follows. We discuss relevant literature in Section \ref{sec:prior_work}, and introduce the preliminaries and problem formulation in Section \ref{sec:problem_statement}. Subsequently, 3D-OGSE algorithm is presented in Section \ref{sec:3d_ogse_motion_planning}, followed by detailed proofs of its probabilistic completeness and guarantee of collision-free trajectory in Section \ref{sec:analysis_3d_ogse}. proves the probabilistic completeness guarantee that the agent converges to the goal position. 
Finally, extensive numerical simulations and comparison with state-of-the-art methods are given in Section \ref{sec:results}, followed by concluding remarks in Section \ref{sec:conclusion}.

\section{\label{sec:prior_work}Related Work}
In this section, we summarize the related work in the areas of path planning and online trajectory generation.


\textbf{Optimization-based methods:} 
Sequential convex programming (SCP) has been used in \cite{Augugliaro} and \cite{Chen} for generating smooth trajectories, while mixed-integer optimization methods have been explored in \cite{Kushleyev, Deits} for generating reliable, collision-free, dynamically feasible trajectories. Several of these methods suffered from higher computational burden.
A minimum snap formulation was presented in \cite{mellinger}  that used differential flatness and pose the trajectory generation problems as Quadratic program (QP) problem. Online 3-D trajectories can be generated using QP methods~\cite{upenn, hkust,ding2019safe_hkust_optimization}. These techniques use convex safe corridors and provide a considerable performance improvement over mixed-integer methods. The trajectory generation problem was formulated as a nonlinear optimization over penalty of smoothness and safety to obtain a locally optimal solution in \cite{zucker}. B-spline optimization was also utilized for generation of dynamically feasible smooth trajectories in \cite{bsplineusenko2017real, zhou2019robust_hkust}.



\textbf{Search-based methods:} Search-based planning methods~\cite{MacAllister, Likhachev, Pivtoraiko} were used for fast generation of trajectories for UAVs. They build a graph by discretizing the state-space of the UAVs, where states form the nodes, and motion primitives form the edges in the graph. These methods rely on searching a large pre-computed lookup table that contains the motion primitives to identify the suitable maneuver. This problem was reduced greatly in \cite{Liu} 
by exploiting differential flatness property of the quadrotor. B-spline-based kinodynamic (EBK) search algorithm was presented in \cite{ding2019efficient_hkust_search} to find a feasible trajectory with minimum control effort. Sequences of axes-aligned cubes in free-space were used in ~\cite{chen2016online} to generate safe corridor, while ~\cite{upenn} leveraged collection of connected convex polyhedrons in the free-space and ~\cite{hkust} used connected spheres of a radius equal to minimum distance to any local obstacle for this purpose. While selection of cubes made the approach in ~\cite{chen2016online} efficient mainly for specific shape of obstacles, selection of restricted radius spheres made the approach in ~\cite{hkust} quite conservative in nature.


\textbf{Sampling-based methods: }Many sampling methods like PRM~\cite{prm} and RRT~\cite{rrt} have been used for collision-free path computation. Sampling-based methods such as~\cite{sertac,lavalle,zinage} avoid the explicit construction of configuration space and have been successfully used for collision-free path computation in high dimensional configuration spaces. These methods compute sampling points in the free space and tend to grow the graph towards the goal configuration. \cite{sertac} propose RRT* and PRM*, which provide asymptotic optimality. Generalized Shape Expansion (GSE)-based algorithm was recently proposed in \cite{zinage,zinage_3d_gse_journal} for 2-D and 3-D environments, respectively, which used a novel notion of generalized shape for fast and efficient free-space exploration to generate a collision-free path. It was found to show superior performance in computational time than several well-established algorithms.
An even faster planner was presented in \cite{zinage_gse_directional_journal} for 2-D environments and in \cite{zinage_gse_directional_journal_3d} for 3-D environments by embedding directional sampling scheme to the GSE for 2-D and 3-D environments, respectively. Path planning in our 3D-OGSE algorithm relies on the generalized shape-driven \cite{zinage_3d_gse_journal} fast sampling-based method, while we use snap minimization~\cite{mellinger} for fast trajectory generation. 

\section{\label{sec:problem_statement}Problem Description and Preliminaries}
\begin{figure*}[]
\captionsetup[subfigure]{justification=centering}
\centering
\begin{subfigure}{0.23\textwidth}
{\includegraphics[scale=0.08]{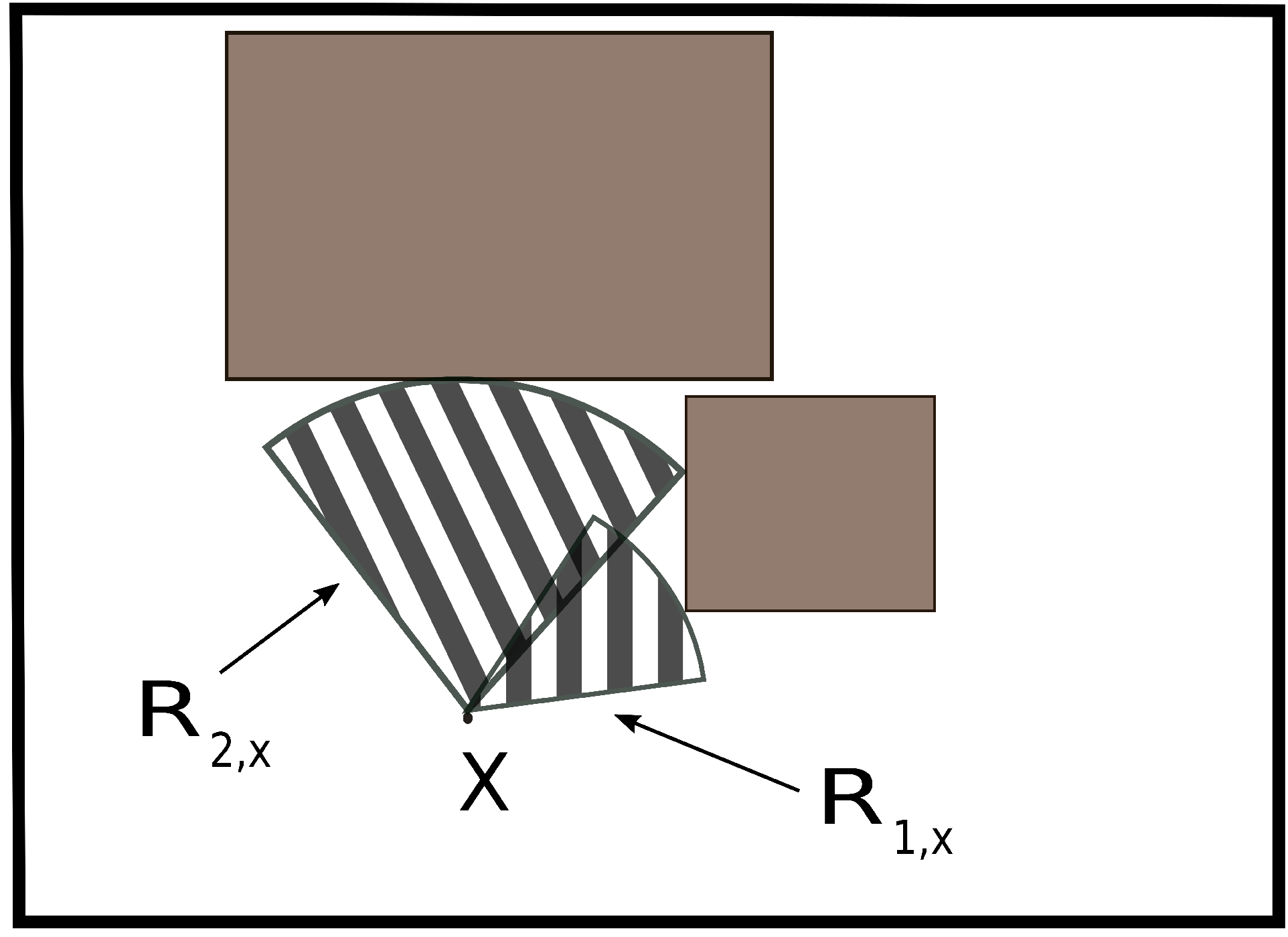}}
\caption{${\mathcal{R}_{1,\boldsymbol{X}}}$ and ${\mathcal{R}_{2,\boldsymbol{X}}}$ }
\label{fig:r1_and_r2}
\end{subfigure}
\begin{subfigure}{0.23\textwidth}
{\includegraphics[scale=0.08]{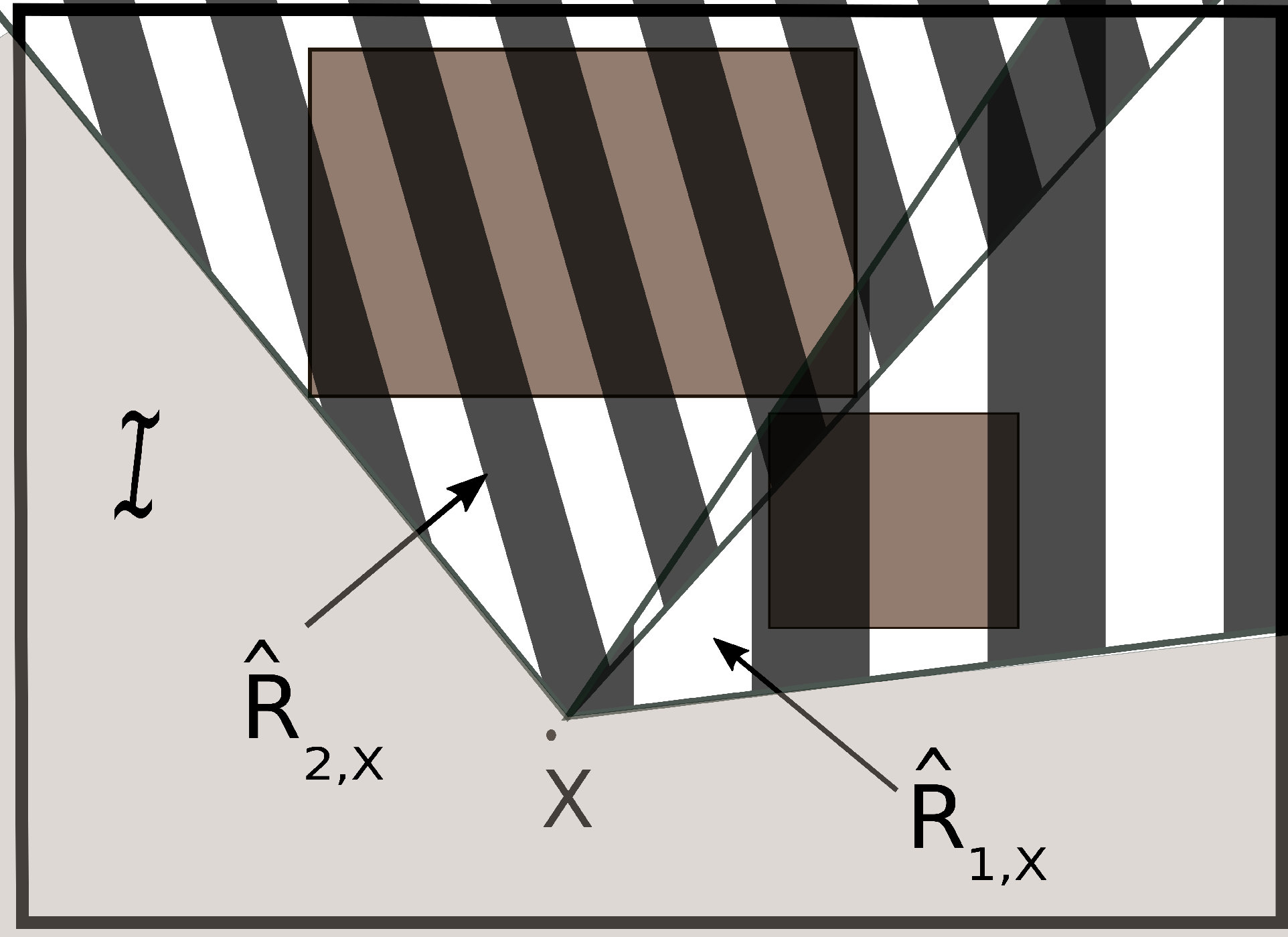}}
\caption{$\widehat{\mathcal{R}_{1,\boldsymbol{X}}}$, $\widehat{\mathcal{R}_{2,\boldsymbol{X}}}$ and $\mathcal{I}$}
\label{fig:r1_and_r2_hats}
\end{subfigure}
\begin{subfigure}{0.23\textwidth}
{\includegraphics[scale=0.08]{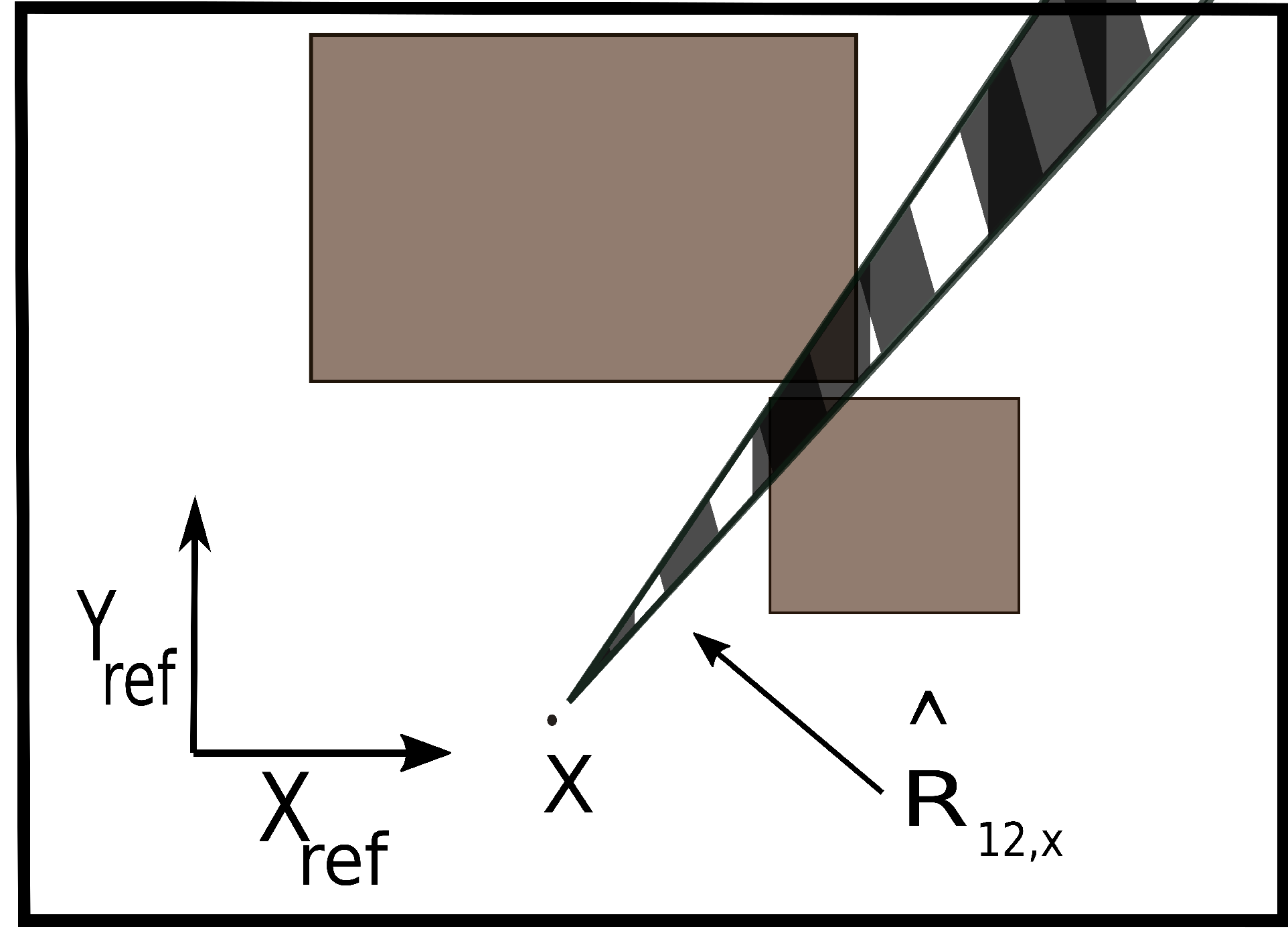}}
\caption{$\widetilde{\mathcal{R}_{12,\boldsymbol{X}}}$}
\label{fig:r12_tilde}
\end{subfigure}
\begin{subfigure}{0.23\textwidth}
{\includegraphics[scale=0.08]{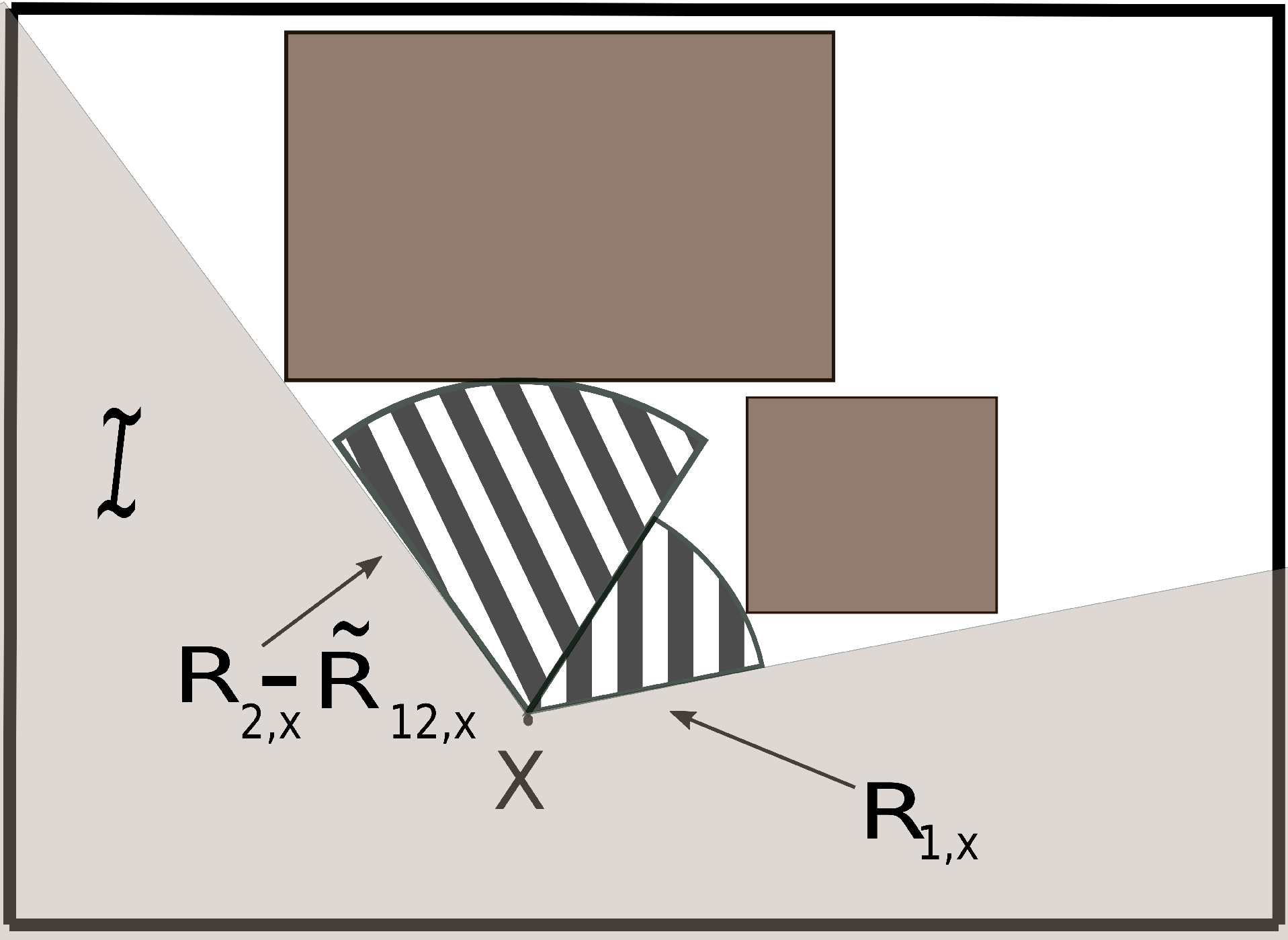}}
\caption{$\mathcal{S}_{\boldsymbol{X}}$}
\label{}
\end{subfigure}
\begin{center}
\captionsetup[subfigure]{justification=centering}
\centering
\end{center}
\caption{An illustration of 2-D projection of 3-D generalized shape $\mathcal{S}_{X}$ about $\boldsymbol{X}$ in an environment having two obstacles.}
\label{fig:gse_2d_shape}
\end{figure*}
\begin{table}[t]
  \centering
   \scalebox{0.99}{
  \begin{tabular}{|p{1cm}||p{7cm}|}
  \hline
 Notations & Description \\
 \hline
 \hline
  $\mathbb{X}, \boldsymbol{X}$ & 3-D workspace and a point in 3-D workspace, respectively \\ 
  \hline
  $\mathbb{V},\mathbb{E}$ & Vertex and edge set, respectively\\    \hline
  $m_\text{loc}$ & Number of obstacles after segmentation of local map in $\mathbb{X}$ \\ 
  \hline 
 $\boldsymbol{p}$ & Current position of agent \\ \hline
 $r_\text{i,X}$ & Minimum distance of $\boldsymbol{X}$ from obstacle $i$ \\ \hline
 $p_\text{i,X}$ & Point on obstacle $i$ at a distance of $r_\text{i,X}$ from $\boldsymbol{X}$ \\ \hline
 $n_\text{i,X}$ & Vector from $\boldsymbol{X}$ to $p_\text{i,X}$  \\ \hline
 $\mathbb{D}_\text{i,X}$ & set of points on obstacle $i$  \\ \hline
 $\mathbb{Q}_{\text{i,X}}$ & points of intersection of plane with normal $\boldsymbol{n}_\text{i,X}$, passing through $\boldsymbol{p}_\text{i,X}$ and the lines joining $\mathbb{D}_\text{i,X}$ and $\mathbf{X}$\\ \hline
 $l_\text{i,X}$ & Maximum distance of point $\boldsymbol{p_{\text{i,X}}}$ from the points in $\mathbb{Q}_{\text{i,X}}$  \\ \hline
 $\mathcal{S}_{\boldsymbol{X}}$ & 3-D generalized shape about $\boldsymbol{X}$  \\
 \hline
 $\mathcal{R}_{i,X}$ &  Spherical sector with vertex at the sampled point $\boldsymbol{X}$ radius equal to $r_{i,X}$, cone angle equal to $\text{tan}^{-1}({l_{i,X}/r_{i,X}})$ and axis along $\boldsymbol{n_\text{i,X}}$ \\ \hline
 $\widehat{\mathcal{R}_{i,X}}$ & Intersection of $\mathbb{X}$ and the spherical sector with vertex at sampled point $\boldsymbol{X}$, cone angle $\texttt{tan}^{-1}(l_{i,X}/r_{i,X})$, axis along $\boldsymbol{n_\text{i,X}}$ and radius being infinite \\ \hline
 $\widetilde{\mathcal{R}_{ij,X}}$ &  $\widehat{\mathcal{R}_{i,X}}\cap\widehat{\mathcal{R}_{j,X}}$ for $i<j$ \\ \hline
  $C_{\boldsymbol{p}}$ & Sensing region about $\boldsymbol{p}$ (assumed to be a cube with center at $\boldsymbol{p}$ and dimensions $10m\times 10m\times 10m$) \\ \hline
 \end{tabular}
 }
 \caption{Notations used in the paper.}
 \label{Tab:Notation}
 \vspace{-10pt}
\end{table}
\subsection{3-D Environment and 3-D Agent}
Consider a 3-D obstacle-cluttered environment denoted by $\mathbb{X}\subset\mathbb{R}^3$. The environment is assumed to be path-connected, that is, there always exists a collision-free path from any point in the free-space $(\mathbb{X_\text{free}})$ to the goal point $\boldsymbol{X}_{\text{goal}}$ with sufficient clearance from obstacle-space $(\mathbb{X_\text{obs}})$. Thus, $\mathbb{X_\text{free}}=\mathbb{X}\setminus\mathbb{X_\text{obs}}$. 
However, the environment is not known \textit{a-priori}. Instead it is assumed that point cloud information about the local environment within the sensing region of on-board sensor of 3-D agent is available with sufficient accuracy for planning of safe trajectory that is dynamically feasible for the 3-D agent. Additionally, we assume that the trajectory planning module of the 3-D agent has access to point cloud data segmented into distinct obstacles in local environment. Several computer vision and geometry processing techniques ~\cite{huang2018recurrent,meng2019vv_segment} perform similar real-time segmentation. The sensing region is assumed to be a cube of dimension $10m\times10m\times10m$ as shown in Fig. \ref{fig:real_traj_1}. We also assume that the 3-D agent can localize itself in the locally available environment information. 

In subsequent discussion we assume the 3-D agent as a point mass. Note that the approach developed in this paper could be easily extended to deal with agent geometry by considering $\mathbb{X_\text{safe}}$ as the available space for collision-free paths instead of $\mathbb{X_\text{free}}$, where $\mathbb{X_\text{safe}}=\mathbb{X_\text{free}}\setminus\delta\mathbb{X_\text{obs}}$, where $\delta\mathbb{X_\text{obs}}$ encapsulates the 3-D agent’s geometric shape.
\subsection{\label{sec:problem_formulation_subsection}Problem Formulation}

Given initial and goal states $\boldsymbol{X}_\text{init},\boldsymbol{X}_\text{goal}\in \mathbb{X}_\text{free}$, respectively, the task is to plan a time-parametrized trajectory $\sigma(t):[t_0,t_f]\longrightarrow\mathbb{R}^3$ based on the locally sensed environment information, until the agent reaches $\boldsymbol{X}_\text{goal}$ starting form $\boldsymbol{X}_\text{init}$, where $t_0$ is the initial time and $t_f$ is the time at which the agent reaches $\boldsymbol{X}_\text{goal}$. Generated trajectory also needs to ensure that the total snap incurred $c(\sigma(t))$ by the 3-D agent is minimized subject to constraints on maximum magnitudes of velocity and acceleration of the agent. Snap minimization takes care of the smoothening of trajectory to make it dynamically feasible to follow for any realistic vehicle \cite{upenn}. The 
problem can be expressed as:
\begin{align}
&\underset{{{\sigma(t)}}}\minimize \;\;  c({\sigma(t)})\nonumber\\ 
& \text{subject to,}\quad{\sigma(t_0)}\;=\;\boldsymbol{X_{\text{init}}};\quad{\sigma(t_f)}\;=\;\boldsymbol{X_{\text{goal}}},\nonumber\\ 
& \|\dot{\sigma}(t)\|\leq {v}_\text{M};\|\ddot{\sigma}(t)\|\leq {a}_\text{M}, {\sigma(t)}\;\in\;\mathbb{X}_{\text{free}};\forall t\in[t_0,t_f)
\label{eq:optimal_motion_planning_problem}
\end{align}

Due to the availability of local environment information only, a systematic re-planning mechanism also needs to be devised for the above trajectory planning.

\subsection{\label{subsec:generalized_shape_gse}Generalized Shape: 2-D GSE and 3-D GSE}
The overall idea of the GSE algorithm in 2-D and 3-D environments is described in ~\cite{zinage}, \cite{zinage_3d_gse_journal}. They present an offline method to generate collision-free feasible paths for an agent in \textit{a-priori} known 2-D and 3-D environments, respectively. Expansion of generated graph over generalized shape, an efficient approximator of safe region, was instrumental for computational advantage of the GSE and 3D-GSE algorithms. For details on the 3D-GSE algorithm, refer to~\cite{zinage_3d_gse_journal}. In the 3D online GSE (3D-OGSE) algorithm presented in this paper, we use the same notion of generalized shape, but present a scheme to extend the 3D-GSE to perform computations in an online manner for unknown environments that is sensed locally only. Consider an agent at a point $\boldsymbol{X}$ in the 3-D environment. Let there be $m_\text{loc}$ obstacles in the workspace. 
For each obstacle $i\in\{1,2,\dots,m_\text{loc}\}$ in the local map, we now have the parameters $r_{\text{i,X}}$, $l_{\text{i,X}}$, and $\boldsymbol{n_\text{i,X}}${ (refer to Table \ref{Tab:Notation})} computed about $\boldsymbol{X}$. Without loss of generality we assume that 
$r_{\text{1,X}}$ $ \leq$ $ r_{\text{2,X}}\dots\leq r_{\text{m}_\text{loc},\text{X}}$. 
The 3-D generalized shape about $\boldsymbol{X}$ is represented as,
\begin{multline}
\mathcal{S}_{\boldsymbol{X}} =\mathcal{R}_{1,X}  \cup (\mathcal{R}_{2,X}-\widetilde{\mathcal{R}_{12,X}}) \cup
(\mathcal{R}_{3,X}-(\widetilde{\mathcal{R}_{23,X}}+\widetilde{\mathcal{R}_{13,X}}))\\
\cup \dots \cup (\mathcal{R}_{m_\text{loc},X}-\cup_{j=1}^{j=i-1}\widetilde{\mathcal{R}_{jm_\text{loc},X}}) \cup \mathcal{I},
\label{eq:shape_set}
\end{multline}
where, the set $\mathcal{I}$ is given as, $\mathcal{I}\subset\mathbb{X}|\mathcal{I}\cap\{\cup_{i=1}^{i=m_\text{loc}}\widehat{\mathcal{R}_{i,X}}\}=\emptyset$. Fig. \ref{fig:gse_2d_shape} shows an illustration of 2-D projection of $\mathcal{S}_{\boldsymbol{X}}$. 

\section{\label{sec:3d_ogse_motion_planning}3-D Online Generalized Shape Expansion (3D-OGSE) Algorithm}
Now, we present 3D-OGSE for online 
trajectory planning having the following steps at $k^{\text{th}}$ planning iteration.
\subsection{\label{subsection:gse_collision_free_path}Generating a Collision-free Path}
The $k^{\text{th}}$ planning iteration is initiated at point $\boldsymbol{p}_{k,0}$, which is $\boldsymbol{X}_{\text{init}}$ for $k=1$ and $\boldsymbol{p}_{k-1,f}$, otherwise. Information about $m_{\text{loc}}$ number of obstacles in the sensing region ($C_{\boldsymbol{p}_{k,0}}$) about $\boldsymbol{p}_{k,0}$ is first obtained from local point cloud data. Because of no information about obstacles outside $C_{\boldsymbol{p}_{k,0}}$, we treat unexplored $\mathbb{X}\setminus C_{\boldsymbol{p}_{k,0}}$ as obstacle-free. With this set-up, a collision-free path from $\boldsymbol{p}_{k,0}$ to $\boldsymbol{X}_{\text{goal}}$ is then generated using the notion of 3-D generalized shape given in \cite{zinage_3d_gse_journal}. Once feasible paths are obtained, shortest path among them is found using Dijkstra's algorithm. This is given in Lines 5 to 7 of Algorithm \ref{alg:3d_online_gse}. The pseudo code of $\texttt{3D-GSE}(\boldsymbol{p}_{k,0},\boldsymbol{X}_{\text{goal}},m_{\text{loc}})$ (Line 5 of Algorithm \ref{alg:3d_online_gse}) for the generation of feasible path is given in Algorithm \ref{alg:3d_gse}. Note that the use of generalized shape facilitates fast path planning. Let the finally obtained path from $\boldsymbol{p}_{k,0}$ to $\boldsymbol{X}_\text{goal}$ generated consist of directed edges through set of vertices $\{\boldsymbol{X}_1,\boldsymbol{X}_2\dots\boldsymbol{X}_{h_k}\}$, where $h_k\in\mathbb{N}$. An overall graph $G$ is updated with the graph generated at $k^{\text{th}}$ planning iteration in Line 6 of Algorithm \ref{alg:3d_online_gse}.

\begin{algorithm}[]
\caption{3D Online GSE Algorithm (3D-OGSE)}
\begin{algorithmic}[1]
\State $k\gets 1\;\;,\mathbb{E}\gets\emptyset\;\;,G_k\gets\emptyset,\;\;\boldsymbol{p}\gets\boldsymbol{p}_{1,0}\gets\boldsymbol{X}_{\text{init}}$
\While {$\|\boldsymbol{p}-\boldsymbol{X}_\text{goal}\|$$\geq\epsilon$}
\State $\boldsymbol{p}_{k,0}\gets\boldsymbol{p},\;\;\mathbb{V}\gets\{\boldsymbol{p}_{k,0},\boldsymbol{X_\text{goal}}\}\;j\gets 1$
\State $m_\text{loc}\gets \texttt{SegmentLocalMap}(\textbf{p})$ // local obstacles
\State $(\mathbb{V},\mathbb{E})\gets \texttt{3D-GSE}(\boldsymbol{p}_{k,0},\boldsymbol{X}_{\text{goal}},m_{\text{loc}})$
\State $G\gets(\mathbb{V},\mathbb{E})\cup G_k$
\State $\texttt{Path}\gets \texttt{djikstra}(\boldsymbol{p}_{k,0},\boldsymbol{X_\text{goal}},G)$
\State $\sigma^k\gets \texttt{TrajOpt}(\texttt{Path},\texttt{Shapes})$
\While{$\texttt{Trigger}(\boldsymbol{p}_{k,0},C_{\boldsymbol{p}_{k,0}},\mathcal{S}_{\boldsymbol{p}_{k,0}},j)=0$}
\State $\boldsymbol{p}\gets \sigma^k(j\Delta t)$        //update agent position
\State $j\gets j+1$
\EndWhile
\State $\boldsymbol{p}_{k,f}\gets\boldsymbol{p}$
\State $GP_k\gets \texttt{Prune}(G,C_{\boldsymbol{p}_{k,0}},\mathcal{S}_{\boldsymbol{p}_{k,0}})$
\State $\mathbb{V}\gets \emptyset,\;\mathbb{E}\gets \emptyset$      // clear vertex and edge Set
\State $G_k\gets G_k\cup GP_k$
\State $k\gets k+1$
\EndWhile
\end{algorithmic}
\label{alg:3d_online_gse}
\end{algorithm}

\begin{algorithm}[t]
\caption{$\texttt{3D-GSE}(\boldsymbol{p}_{k,0},\boldsymbol{X}_{\text{goal}},m_{\text{loc}})$}
\begin{algorithmic}[1]
\State $\mathbb{V}\gets\{\boldsymbol{p}_{k,0},\boldsymbol{X_\text{goal}}\}$
\While {No connected graph is generated from $\boldsymbol{p}_{k,0}$ to $\boldsymbol{X_\text{goal}}$}
\State $\boldsymbol{X_{\text{rand}}}\gets \texttt{SamplePoint}$
\State $\boldsymbol{X_\text{nearest}}\gets\texttt{Nearest}(\mathbb{V},\boldsymbol{X_\text{rand}})$
\State $\boldsymbol{X}_\text{new}\gets \texttt{Steer}(\boldsymbol{X_\text{rand}},\boldsymbol{X_\text{nearest}})$ 
\State $\mathbb{V}\gets\mathbb{V}\cup\{\boldsymbol{X}_\text{new}\}$
\State $\mathbb{E}\gets\mathbb{E}\cup\{(\boldsymbol{X_\text{nearest}},\boldsymbol{X}_\text{new})\}$
\State $\mathbb{X_{\text{near}}} \gets \texttt{NearIntersectedShapes}(G,\boldsymbol{X_{\text{new}}})$
\For{$\boldsymbol{X_{\text{n}}}\in \mathbb{X_{\text{near}}}$}
\State $\mathbb{E}\gets \mathbb{E} \cup \{(\boldsymbol{X_{\text{n}}},\boldsymbol{X_{\text{new}}})\}$
\EndFor
\EndWhile
\end{algorithmic}
\label{alg:3d_gse}
\end{algorithm}


\subsection{Trajectory Optimization}\label{sec:trajopt}
The collision-free path from $\boldsymbol{p}_{k,0}$ to $\boldsymbol{X}_\text{goal}$ obtained in Section \ref{subsection:gse_collision_free_path} is piece-wise linear between the set of vertices $\{\boldsymbol{X}_1,\boldsymbol{X}_2\dots\boldsymbol{X}_{h_k}\}$. In order to obtain a dynamically feasible smooth trajectory along this path that could be followed by a realistic vehicle, inspired by \cite{mellinger}, a QP for snap minimization is framed using time-parametrized polynomial in \eqref{optimization_problem} subject to the constraints \eqref{eqn:first_boudary}-\eqref{eqn:velocity_and_acceleration_constraints}. Thus, a smooth and locally collision-free trajectory $\sigma^k$ is obtained (see the function $\texttt{TrajOpt}$ in Line 8 of Algorithm \ref{alg:3d_online_gse}).
The snap minimization problem then becomes:
\begin{subequations}
\begin{align}
     & \underset{\sigma^k(t)}\argmin\quad J=\sum_{i=1}^{l_k}\int_{t_i}^{t_{i+1}}\| \frac{d^4}{dt^4}\sigma^k_i(t)\|^2,\;\;\text{subject to,}\label{optimization_problem}\\
       & {\text{Boundary conditions:}}\;\;\boldsymbol{p}_{1,0}=\boldsymbol{X}_{\text{init}},\;\boldsymbol{p}_{k,0}=\boldsymbol{p}_{k-1,f}\label{eqn:first_boudary}\\
        & \sigma^k(t_{l_k})=\boldsymbol{X}_{\text{goal}},\;\sigma^k(t_1)=\boldsymbol{p}_{k,0}\\
            & \frac{d}{dt}\sigma^k(t)\big|_{\sigma^k(t_1)=\boldsymbol{p}_{k,0}}= \frac{d}{dt}\sigma^{k-1}(t)|_{\sigma^{k-1}(t)=\boldsymbol{p}_{k-1,f}}\quad\label{eqn:replanning_smoothness constraint}\\
    & \text{Other constraints for all }i=1,2,\dots,l_k:\nonumber\\
     & \frac{d^z}{dt^z}\sigma^k_i(t_i )=\frac{d^z}{dt^z}\sigma^k_{i+1}(t_{i}) \quad z=0,\dots,4\label{eqn:four_derivatives_smoothness}\\
     & \|\frac{d}{dt}\sigma^k_{i}(t)\|\leq {v}_\text{M},\;\|\frac{d^2}{dt^2}\sigma^k_{i}(t)\|\leq {a}_\text{M}, \;\;  \sigma^k_i(t)\in \mathcal{S}_{\boldsymbol{X}_i}\label{eqn:velocity_and_acceleration_constraints}
\end{align}
\end{subequations}
Thus, the generated trajectory is a collection of segments as $\sigma^k(t)=\{\sigma^k_i(t)=\sum\limits_{\ell=1}^{n} \sigma^k_{i\ell}(t)^\ell\;\text{for}\;{t_i\leq t< t_{(i+1)}}\}_{i=1}^{l_k}$. For each $i=\{1,\dots,l_k\}$,   $\sigma^k_i$ connects $\boldsymbol{X}_i$ to $\boldsymbol{X}_{i+1}$ and belongs to $\mathcal{S}_{\boldsymbol{X}_i}$ (see \eqref{eqn:velocity_and_acceleration_constraints}) ensuring that $\sigma^k_i$ lies in $\mathbb{X}_{\text{free}}$. As a consequence, inside $k^{\text{th}}$ planning horizon, each points on $\sigma^k$ also lies in $\mathbb{X}_{\text{free}}$. It also facilitates in fast QP solving by restricting the search-space to $\mathcal{S}_{\boldsymbol{X}_i}$. Eq. \eqref{eqn:four_derivatives_smoothness} and velocity and acceleration constraints in \eqref{eqn:velocity_and_acceleration_constraints} provide the dynamical feasibility to $\sigma^k$. Moreover, Eqs. \eqref{eqn:first_boudary} - \eqref{eqn:replanning_smoothness constraint} renders smoothness to trajectory even at the transition point from $(k-1)^{\text{th}}$ to $k^{\text{th}}$ planning iteration. 

\begin{algorithm}[]
\caption{$\texttt{Trigger}(\boldsymbol{p}_{k,0},C_{\boldsymbol{p}_{k,0}},\mathcal{S}_{\boldsymbol{p}_{k,0}},j)$}
\begin{algorithmic}[1]
\State $\boldsymbol{p}'\gets \sigma^k(j\Delta t)$ 
\If {$\boldsymbol{p}'\in C_{\boldsymbol{p}_{k,0}}\quad\text{and}\quad \boldsymbol{p}'\in \mathcal{S}_{\boldsymbol{p}_{k,0}}$}
\State $\texttt{Trigger}(\boldsymbol{p}_{k,0},C_{\boldsymbol{p}_{k,0}},\mathcal{S}_{\boldsymbol{p}_{k,0}},j)=0$
\Else
\State $\texttt{Trigger}(\boldsymbol{p}_{k,0},C_{\boldsymbol{p}_{k,0}},\mathcal{S}_{\boldsymbol{p}_{k,0}},j)=1$
\EndIf
\State $\text{return}\quad\texttt{Trigger}(\boldsymbol{p}_{k,0},C_{\boldsymbol{p}_{k,0}},\mathcal{S}_{\boldsymbol{p}_{k,0}},j) $
\end{algorithmic}
\label{alg:trigger_replanning_gse}
\end{algorithm}

\begin{figure}[t!]
\centering
\begin{subfigure}[b]{0.3\textwidth}
{\centerline{\includegraphics[width=0.911\textwidth,height=3.9453cm]{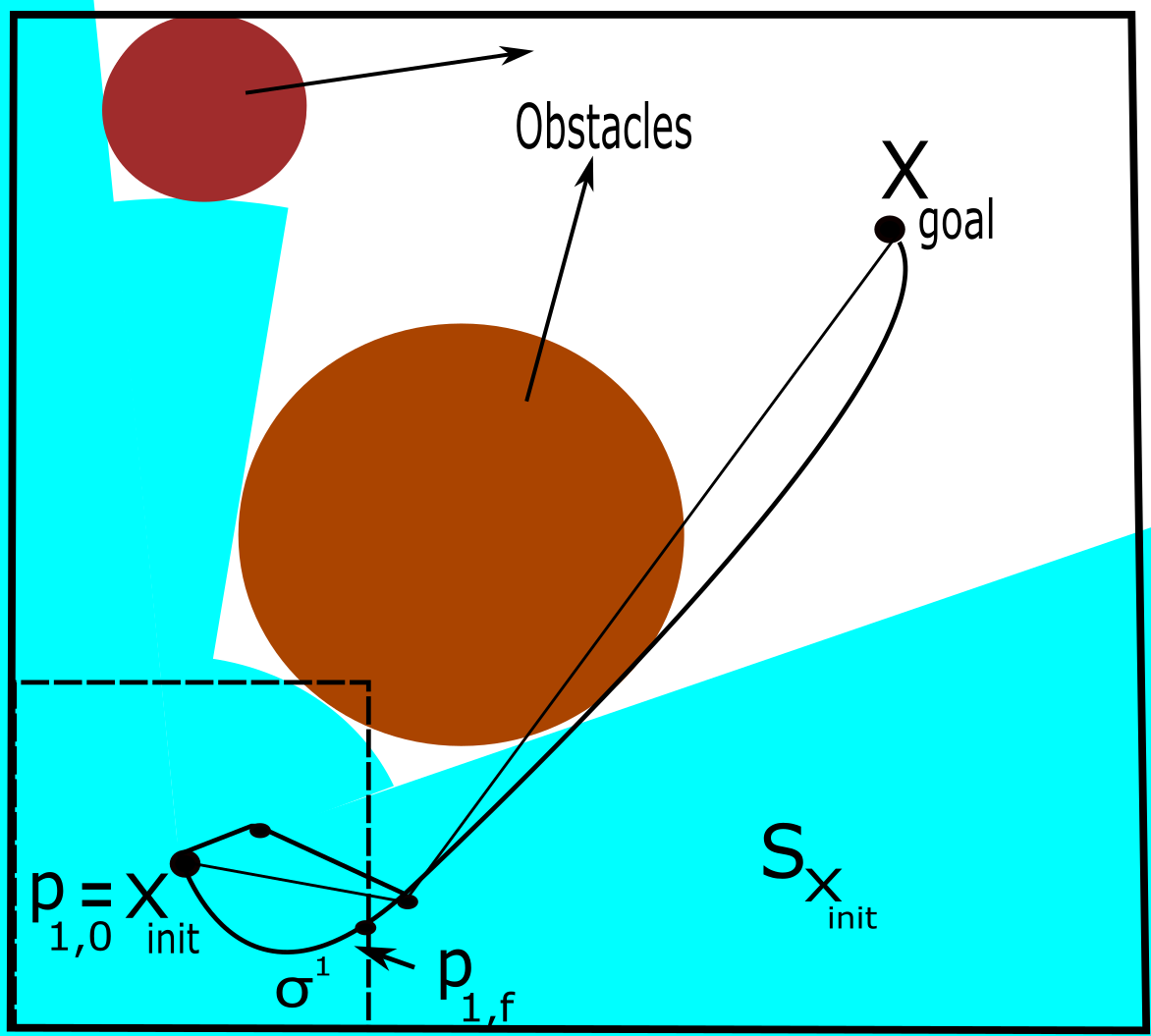}}}
\end{subfigure}
\caption{Illustration of first planning iteration (2-D projection). The area inside the dotted square shows projection of the sensing region $C_{\boldsymbol{X}_{\text{init}}}$ about ${\boldsymbol{X}_{\text{init}}}$, and the blue region shows the generalized shape $\mathcal{S}_{\boldsymbol{X}_{\text{init}}}$ about ${\boldsymbol{X}_{\text{init}}}$. The agent moves along the smooth trajectory $\sigma^1$ until it reaches the boundary of sensing shape about $\boldsymbol{X}_{\text{init}}$ (=$C_{\boldsymbol{X}_{\text{init}}}\cap \mathcal{S}_{\boldsymbol{X}_{\text{init}}}$). }
\label{fig:iteration_gse_sensing_shape}
\vspace{-10pt}
\end{figure}
\subsection{Replanning Trigger for RHP Mechanism}\label{subsec:replanningtrigger}

Sensing shape about a point $\boldsymbol{X}$ is defined as the region $C_{\boldsymbol{X}}\cap\mathcal{S}_{\boldsymbol{X}}$, where $C_{\boldsymbol{X}}$ and $\mathcal{S}_{\boldsymbol{X}}$ are sensing region and generalized shape, respectively, about $\boldsymbol{X}$. Starting from $\boldsymbol{p}_{k,0}$, as the 3-D agent moves along $\sigma^k(t)$, at every time-instant it checks whether the next point on $\sigma^k$ lies inside the sensing shape about $\boldsymbol{p}_{k,0}$. This check helps in reducing computational overhead since $C_{\boldsymbol{p}_{k,0}}$ and $\mathcal{S}_{\boldsymbol{p}_{k,0}}$ are already computed at the beginning of the $k^{th}$ planning iteration. The agent keeps on following $\sigma^k$ if the next point lies inside the sensing shape about $\boldsymbol{p}_{k,0}$. A re-planning is triggered only when the agent reaches some point $\boldsymbol{p}_{k,f}$ on the boundary of $C_{\boldsymbol{p}_{k,0}}\cap\mathcal{{S}}_{\boldsymbol{p}_{k,0}}$. Once a re-planning is triggered the RHP mechanism shifts the planning horizon to $(k+1)^{\text{th}}$ planning iteration (restricted by $C_{\boldsymbol{p}_{k,0}}$) such that $\boldsymbol{p}_{k+1,0}=\boldsymbol{p}_{k,f}$ is satisfied.

The fore-mentioned RHP scheme helps in avoiding re-planning trajectory in every time-point, thus saving much computational time. Clearly, this RHP scheme is easily real-time applicable for stationary obstacles scenario. However, for dynamic obstacles scenario the re-planning rate might increase based on moving obstacles' states in neighborhood of the sensing shape of $\boldsymbol{p}_{k,0}$. Even in that scenario also because of significantly low computational time requirement (see Table II) of the presented 3D-OGSE algorithm, it would not add significantly high computational overhead and thus would be real-time applicable as well. 
\section{\label{sec:analysis_3d_ogse}Analysis of the 3D-OGSE Algorithm}
\subsection{Probabilistic Completeness of the 3D-OGSE Algorithm}
The probabilistic completeness proof for the offline GSE algorithm is given in \cite{probabilistic_completeness_gse}. However in the following proof, the local sensing region is accounted for.

Let the region of intersection of the sensing region $C_{\boldsymbol{X}}$ and generalized shape $\mathcal{S}_{\boldsymbol{X}}$ i.e. $C_{\boldsymbol{X}}\cap\mathcal{S}_{\boldsymbol{X}}$ be termed as the \say{Sensing Shape} about $\boldsymbol{X}$. Consider a sensing restricted visibility function $g: 2^{\mathbb{X}_{\text{free}}}\rightarrow 2^{\mathbb{X}_{\text{free}}}$, defined as follows:
\begin{align} 
   g(x) = \{\boldsymbol{X} \in \mathbb{X}_{\text {free }}\mid C_{\boldsymbol{X}} \cap \mathcal{S}_{\boldsymbol{X}}\cap x \ne \emptyset\}
   \label{eq:visb_fn}
\end{align}
where, $2^{\mathbb{X}_{\text{free}}}$ is the power set of the free-space. Thus, this function maps from the power set of the free-space to itself and provides the set of points that contain any point of the argument in the sensing shape. 

\begin{figure*}[]
\captionsetup[subfigure]{justification=centering}
\begin{subfigure}{0.2\textwidth}
{\includegraphics[scale=0.108]{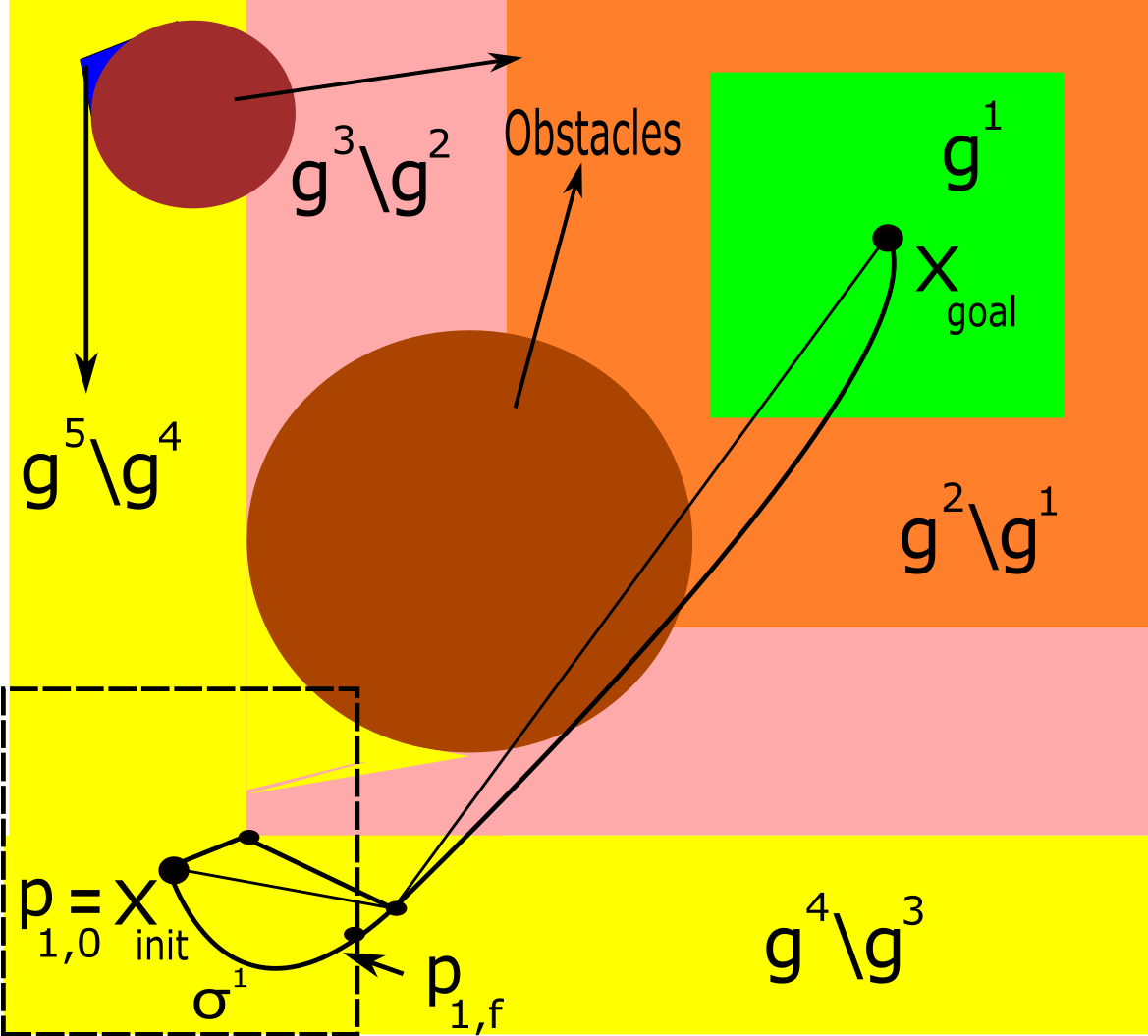}}
\caption{Planning iteration 1 ($k=1$)}
\label{fig:iteration_1}
\end{subfigure}
\begin{subfigure}{0.2\textwidth}
\includegraphics[scale=0.108]{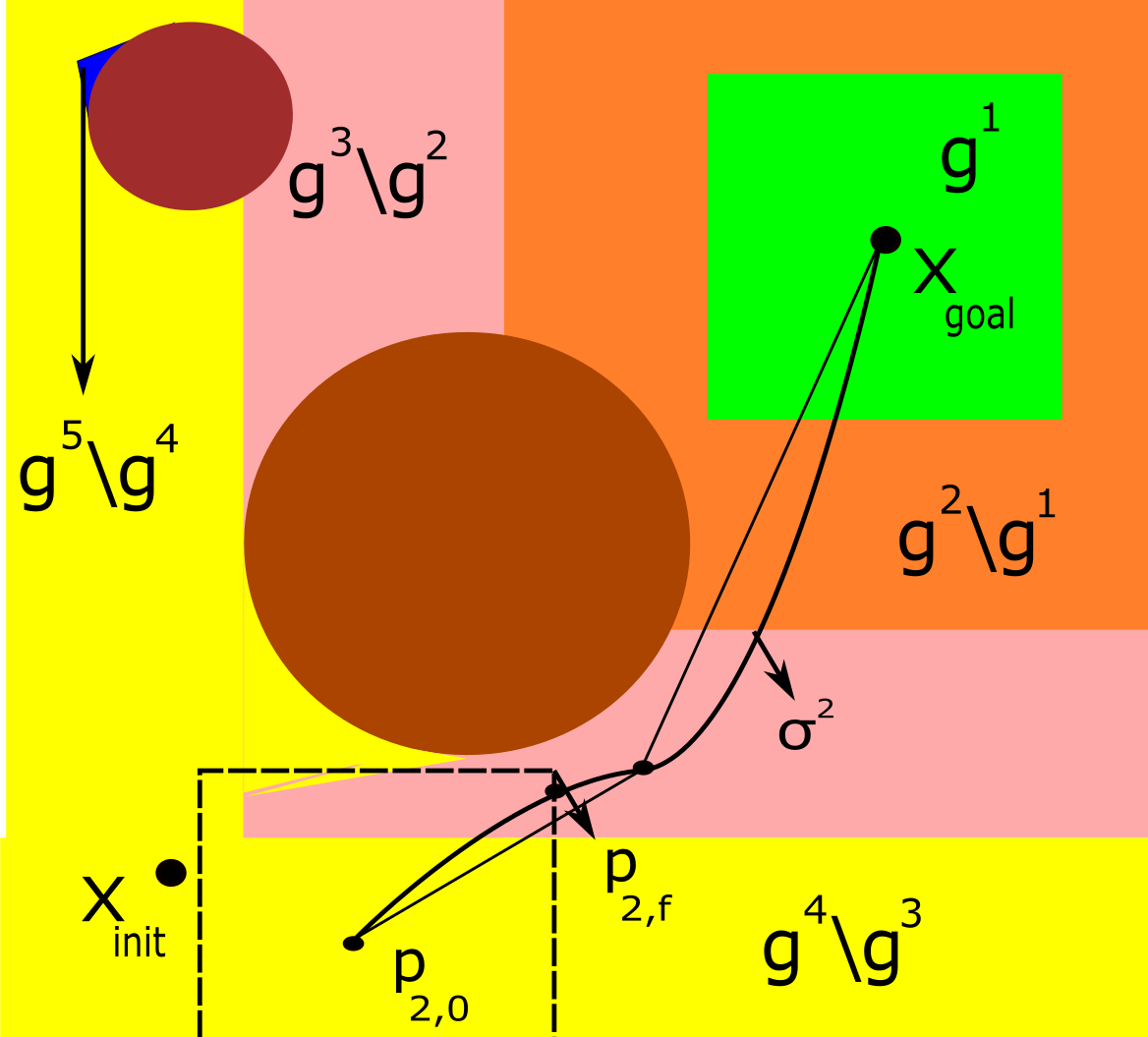}
\caption{Planning iteration 2 ($k=2$)}
\label{fig:iteration_2}
\end{subfigure}
\begin{subfigure}{0.2\textwidth}
\includegraphics[scale=0.108]{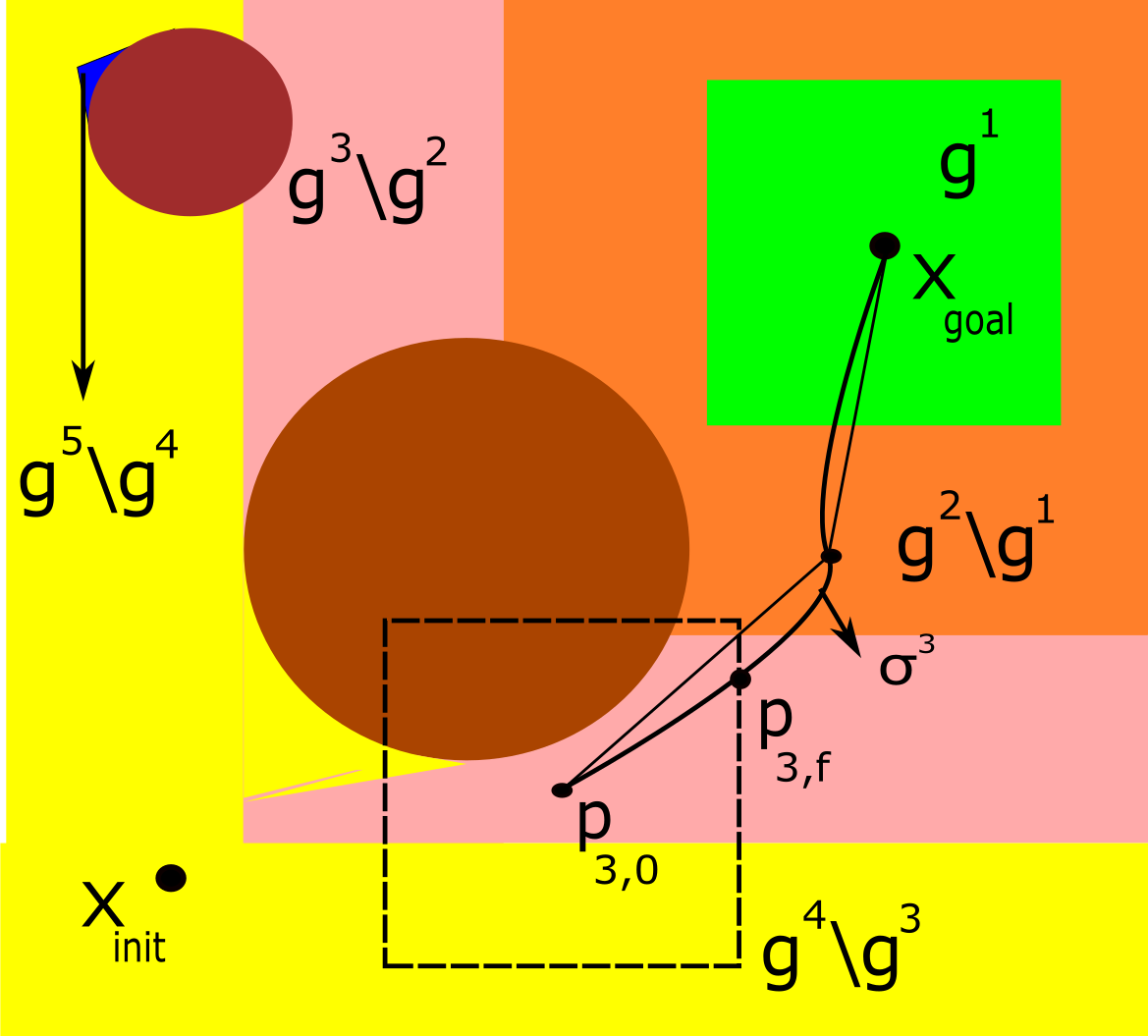}
\caption{Planning iteration 3 ($k=3$)}
\label{fig:iteration_3}
\end{subfigure}
\begin{subfigure}{0.2\textwidth}
\includegraphics[scale=0.108]{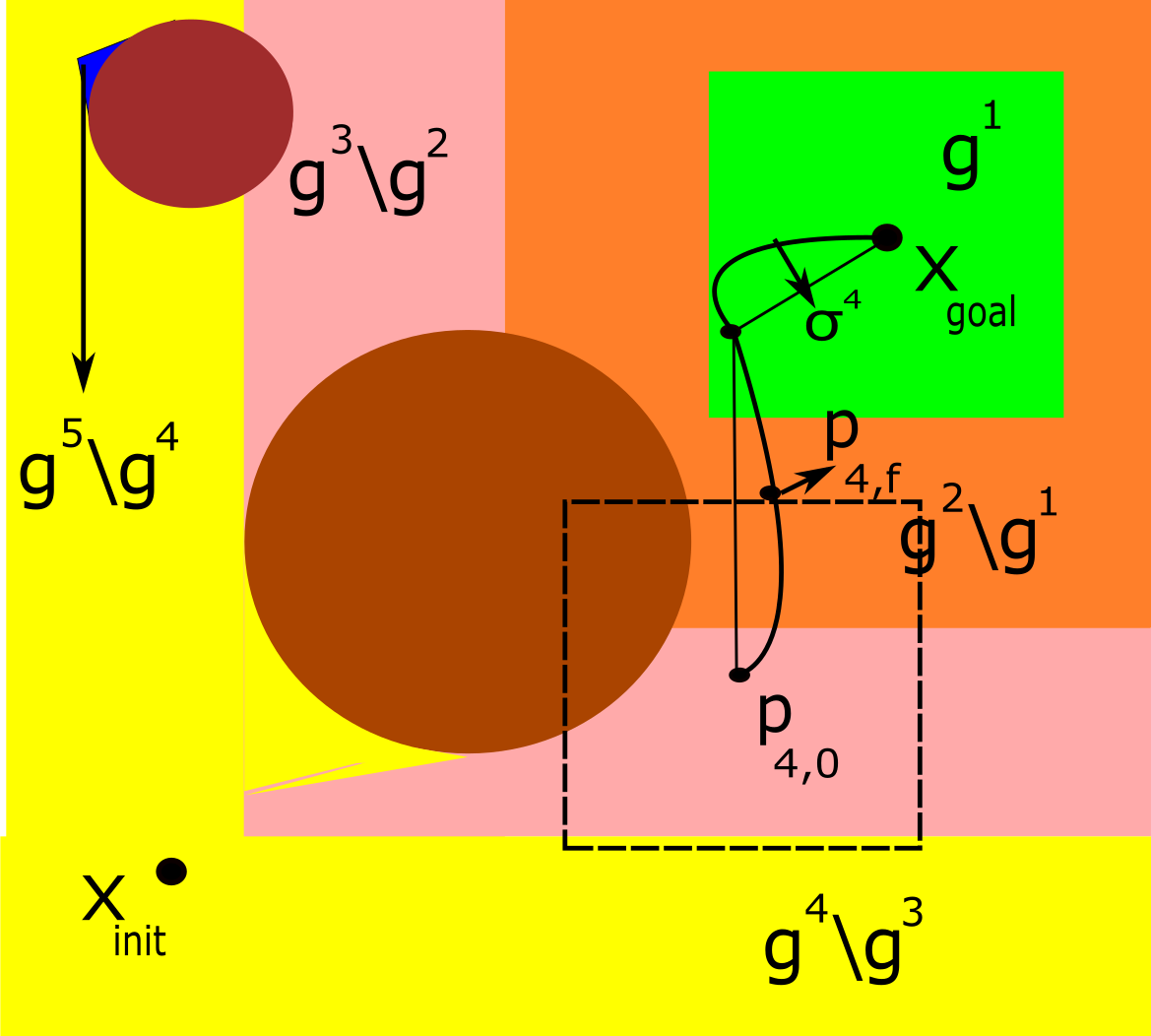}
\caption{Planning iteration 4 ($k=4$)}
\label{fig:iteration_4}
\end{subfigure}
\begin{subfigure}{0.17\textwidth}
\includegraphics[scale=0.108]{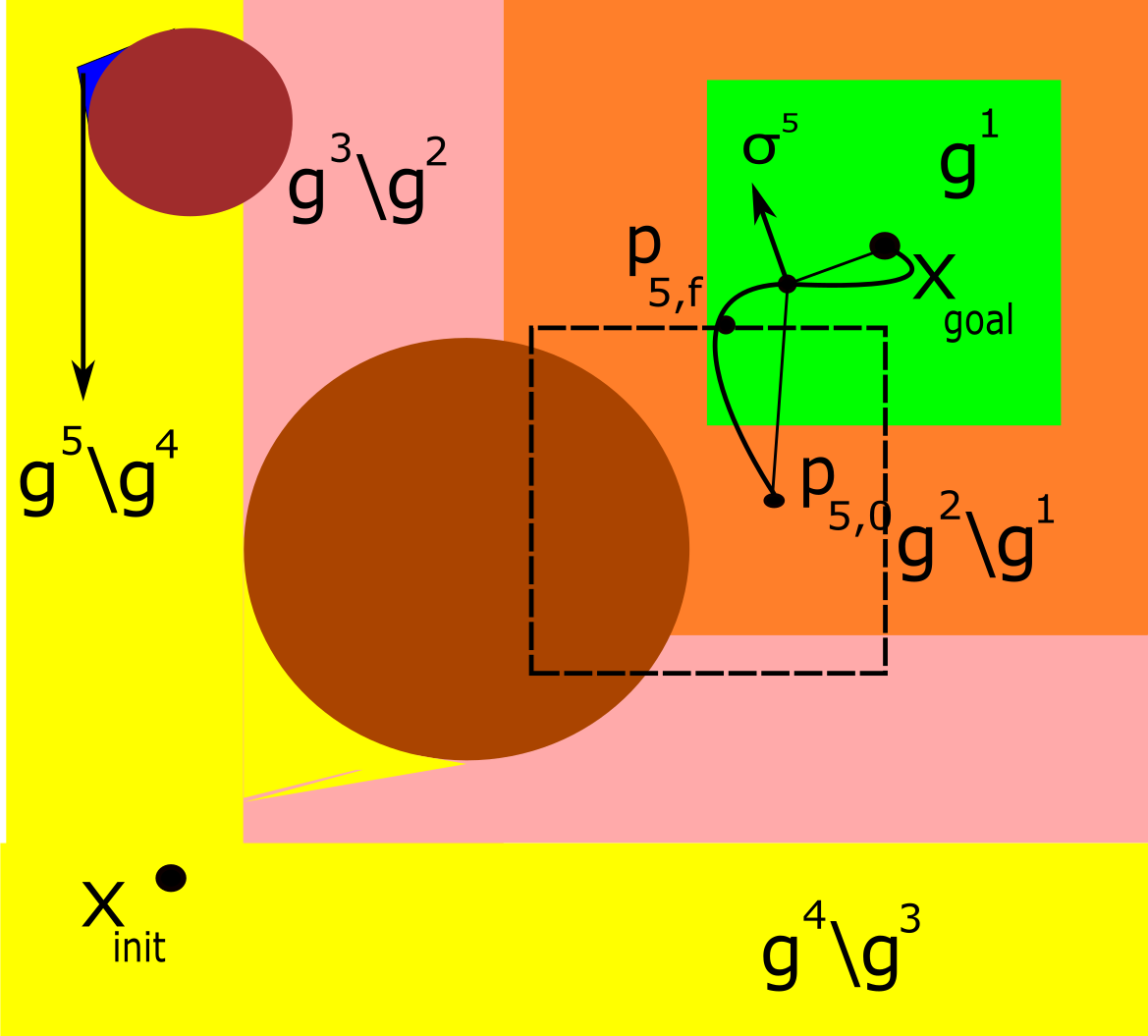}
\caption{Planning iteration 5 ($k=5$)}
\label{fig:iteration_4}
\end{subfigure}
\caption{Illustration of planning iterations of the 3D-OGSE algorithm for dynamically feasible safe trajectory generation (2-D projection). Here, $g^{i}\triangleq g^{(i)}(\boldsymbol{X}_{\text{goal}})\forall i\in\{1,2\dots 5\}$ denotes $i^{\text{th}}$ iterate of sensing-restricted visibility function. (a) In planning iteration 1, the agent starts from $\boldsymbol{p}_{1,0}=\boldsymbol{X}_\text{init}\in g^4$ and generates a locally safe and smooth trajectory $\sigma^1$ and moves along it until it reaches a point $\boldsymbol{p}_{1,f}$ on the boundary of sensing shape about $\boldsymbol{p}_{1,0}$ (=$C_{\boldsymbol{p}_{1,0}}\cap \mathcal{S}_{\boldsymbol{p}_{1,0}}$). 
(b)-(e) In planning iteration 'k' for $k=2,3,4,5$, the agent starts from $\boldsymbol{p}_{k,0}=\boldsymbol{p}_{k-1,f}$ and generates a locally safe and smooth trajectory $\sigma^1$ and moves along it until it reaches a point $\boldsymbol{p}_{k,f}$ on the boundary of sensing shape about $\boldsymbol{p}_{k,0}$ (=$C_{\boldsymbol{p}_{k,0}}\cap \mathcal{S}_{\boldsymbol{p}_{k,0}}$). During this time, in planning iterations 2, 4 and 5 the agent moves from higher iterate of $g$ to lower iterate of $g$ (from $g^4$ to $g^3$, from $g^3$ to $g^2$, and from $g^2$ to $g^1$, respectively). Finally, at the end of planning iteration 5 the agent reaches $\boldsymbol{p}_{5,f}\in g(\boldsymbol{X}_{\text{goal}})$, that is $\boldsymbol{X}_{\text{goal}}\in C_{\boldsymbol{p}_{1,0}}\cap \mathcal{S}_{\boldsymbol{p}_{1,0}}$ ensuring existence of a straight-line edge connectivity between $\boldsymbol{p}_{5,f}$ and $\boldsymbol{X}_{\text{goal}}$.
}
\label{fig:iterations_gse_algorithm}
\end{figure*}


The main intuition for proposing the following two lemmas is as follows. As the agent moves in the free space to reach the goal position, the union of the sensing shapes of the agent grows in volume and at some iteration this union covers the entire free space and it also contains the goal point $\boldsymbol{X}_{\text{goal}}$. At this stage one can claim that there is a collision-free path from the agent's current position to the goal position.

Let the function $g^{(2)}(\boldsymbol{X}_{\text{goal}})\triangleq g(g(\boldsymbol{X}_{\text{goal}}))$. In the following two lemmas, we prove that the function $\mu({{g}^{(i)}({\boldsymbol{X}_\text{goal}})})$ increases by a positive constant with every $i$ and equals $\mu(\mathbb{X}_{\text{free}})$ for some value of $i\geq q$. Consequently, this would imply that the entire free-space $\mathbb{X}_{\text{free}}$ is also contained in ${{g}^{(i)}({\boldsymbol{X}_\text{goal}})}$ for the corresponding finite value of $i$.


\begin{lemma}
\textit{For $g$ defined in \eqref{eq:visb_fn}, for all $i$, ${g}^{(i)}({\boldsymbol{X}_\text{goal}})\subseteq {g}^{(i+1)}({\boldsymbol{X}_\text{goal}})$}
\label{lem:f expanding}
\end{lemma}
\begin{proof}
By definition of $g$, clearly, ${g}^{(i+1)}({\boldsymbol{X}_\text{goal}})$ includes all points $\boldsymbol{X}\in\mathbb{X}_{\text{free}}$ such that $C_{\boldsymbol{X}}\cap \mathcal{S}_{\boldsymbol{X}} \cap{g}^{(i)}({\boldsymbol{X}_\text{goal}})\ne\emptyset$. Now, consider any point $\boldsymbol{X}\in {g}^{(i)}({\boldsymbol{X}_\text{goal}})$. Recall from the definition of the sensing shape that $\mathbf{X}\in {C_{\boldsymbol{X}}}\cap\mathcal{S}_{\boldsymbol{X}}$. This implies that $C_{\boldsymbol{X}}\cap\mathcal{S}_{\boldsymbol{X}} \cap{g}^{(i)}({\boldsymbol{X}_\text{goal}})\supseteq\{\mathbf{X}\}$. Hence, the lemma follows.
\end{proof}

\begin{lemma}
\textit{For $g$ defined in Eqn. \eqref{eq:visb_fn}, $\mu({g}^{(i+1)}({\boldsymbol{X}_\text{goal}})\setminus{g}^{(i)}({\boldsymbol{X}_\text{goal}}))>c$ for a fixed point positive constant $c$, if ${g}^{(i)}({\boldsymbol{X}_\text{goal}}))\subset\mathbb{X}_{\text{free}}$.} 
\label{lem: f expanding at min rate}
\end{lemma}
\begin{proof}
Consider any $i$ such that ${g}^{(i)}({\boldsymbol{X}_\text{goal}}))\subset\mathbb{X}_{\text{free }}$. Let the union of $\epsilon$-neighbourhoods of all points within the set ${g}^{(i)}({\boldsymbol{X}_\text{goal}})$ be denoted by $E_i$. From Lemma \ref{lem:f expanding}, it follows that there exists a fixed positive $\epsilon$ such that $((E_i\setminus{g}^{(i)}({\boldsymbol{X}_\text{goal}}))\setminus\mathbb{X}_{\text{obs}})\subseteq{g}^{(i+1)}({\boldsymbol{X}_\text{goal}})$. We can then conclude that $\mu({g}^{(i+1)}({\boldsymbol{X}_\text{goal}})\setminus{g}^{(i)}({\boldsymbol{X}_\text{goal}})>\mu((\partial{g}^{(i)}({\boldsymbol{X}_\text{goal}}))*\epsilon $ where $\partial \mathbb{S}$ denotes the boundary of a set $\mathbb{S}$. Since the free-space is assumed to be path-connected, there exist no isolated points surrounded by obstacles in $\mathbb{X}_{\text{free}}$. This guarantees a lower bound on the rate of growth of the volume of ${g}^{(i)}({\boldsymbol{X}_\text{goal}})$, and proves the lemma.
\end{proof}

From Lemma \ref{lem: f expanding at min rate}, measure of $g$ grows by at least $c$ at every iteration $i$ of the sensing restricted visibility function $g$. Considering that the free-space $\mathbb{X}_{\text{free}}$ is bounded, thus, Lemmas \ref{lem:f expanding} and \ref{lem: f expanding at min rate} imply that for sufficiently large value of $i$, $\mu({g}^{(i)}({\boldsymbol{X}_\text{goal}}))=\mu(\mathbb{X}_{\text{free}})$. Hence, we define a sensing-restricted visibility index as follows,
\begin{equation}
     q=\underset{i \in \mathbb{N}}{\argmin}(\mu({g}^{(i)}({\boldsymbol{X}_\text{goal}}))=\mu(\mathbb{X}_{\text{free}}))
     \label{eq:visibindex}
\end{equation}
Due to construction of the sensing shape above, from Eqn. \eqref{eq:visb_fn} it can be shown easily that $g^{(q)}(\boldsymbol{X}_{\text{goal}})=\mathbb{X}_{\text{free}}$.
\begin{theorem}
\textit{The probability that the 3D-OGSE algorithm fails to generate a feasible path from the current position $\boldsymbol{X}_{\text{init}}$ to $\boldsymbol{X}_{\text {goal}}$ after $N$ planning iterations is at most $a e^{-b N},$ for some positive real numbers $a$ and $b$. This ensures a probabilistic guarantee that a collision-free path is generated from $\boldsymbol{X}_{\text{init}}$ to $\boldsymbol{X}_{\text{goal}}$ over multiple planning iterations using the 3D-OGSE algorithm.}
\label{theorem:probabilistic_completeness}
\end{theorem}
\begin{proof}
Let $\widehat{\mathbb{V}}_k$ ($G_k=(\widehat{\mathbb{V}}_k,\mathbb{E}_k)$ in Line 16 of Algorithm \ref{alg:3d_online_gse} denote the set of vertices excluding the goal point $\boldsymbol{X}_{\text{goal}}$ generated by the 3D-OGSE algorithm after $k$ planning iterations. Note that the 3D-OGSE algorithm succeeds generating a collision-free path after $N$ planning iterations if $\widehat{\mathbb{V}}_N\cap g(\boldsymbol{X}_{\text{goal}})\neq\emptyset$. 

From construction of $G_k$ (Line 14 and 16 of Algorithm \ref{alg:3d_online_gse}) all the vertices in $\widehat{\mathbb{V}}_k$ are guaranteed to lie in the free-space. Now, define $i_k$ to be the minimum value of $i$ such that $\widehat{\mathbb{V}_k}\cap{g}^{(i)}({\boldsymbol{X}_\text{goal}})\neq\emptyset$. In addition, define a random variable $A_k$, which is equal to one if ${i}_{k+1}<i_k$ or zero otherwise. $A_k=1$ indicates that some of the additional vertices of the extended graph generated at the $(k+1)^{\text{th}}$ planning iteration lies in the set $g^{(i)}(\boldsymbol{X}_{\text{goal}})$ with a lower $i$ than before. Now, let us consider a random variable $A$ which is equal to $A_1+A_2\dots+A_N$. From the definition of $q$ in Eqn. \eqref{eq:visibindex}, clearly, the maximum value of $A$ is $q$.

Consider a random variable $\mathcal{P}_N$ such that $\mathcal{P}_N=1$ denotes the event of success, that is, a collision-free path is generated from initial to goal point after $N^{\text{th}}$ planning iteration, and $\mathcal{P}_N=0$ denotes the event of failure, that is, even after $N^{\text{th}}$ planning iteration a collision-free path is not generated from initial to goal point.

If $A_k=1$ at least $q$ times for various $k\leq N$, then the graph generated till the $N^\text{th}$ planning iteration of the 3D-OGSE algorithm has progressed from ${g}^{(q)}({\boldsymbol{X}_\text{goal}})$ to $g({\boldsymbol{X}_\text{goal}})$. Hence, we have $\widehat{\mathbb{V}}_N\cap {g}({\boldsymbol{X}_\text{goal}})\neq\emptyset$, which indicates that $\mathcal{P}_N=1$ implying success in $N^\text{th}$ planning iteration. However, it could be noted that success could also be achieved even for $A<q$. Therefore, it can be concluded that $\{\mathcal{P}_N=1\}\supseteq\{A=q\}$. Taking complement on both sides gives $\{\mathcal{P}_N=0\}\subset\{A<q\}$. 
\begin{prop}
\textit{For $k^\text{th}$ planning iteration of 3D-OGSE, $\mathbb{P}(A_k = 1)$ is bounded below by a positive fraction $p<1$ for all $k$ such that $\widehat{\mathbb{V}}_k\cap {g}({\boldsymbol{X}_\text{goal}})=\emptyset$.}
\label{prop:parameter_p_berna}
\end{prop}
\begin{proof}
At the $k^{\text{th}}$ iteration, let there be $w$ vertices in $\widehat{\mathbb{V}}_k$, which belong to ${g}^{(i_k)}({\boldsymbol{X}_\text{goal}})$ also. It is needed to be shown here that there exists a minimum probability of extending the graph in the next planning iteration such that it would contain vertices belonging to ${g}^{(i_k-1)}({\boldsymbol{X}_\text{goal}})$. Clearly, if $w$ is more, then the probability of getting vertices in the $(k+1)^{\text{th}}$ planning iteration that belongs to the region ${g}^{(i_k-1)}({\boldsymbol{X}_\text{goal}})$ also increases. In other words, as $w$ increases, the probability that $A_k=1$ also increases. Hence, considering $w=1$ suffices for this proof.

Now, for $w=1$, let the corresponding vertex be denoted by $\boldsymbol{v}$. Let $\text{Vor}(\boldsymbol{v})$ denote the Voronoi cell of the vertex $\boldsymbol{v}$ in the Voronoi diagram consisting of $|{\mathbb{V}}_k|$ vertices of the graph $G_k$ generated by the 3D-OGSE so far. Because of the finite number of vertices in $G_k$, for the vertex $\boldsymbol{v}\in {g}^{(i_k)}({\boldsymbol{X}_\text{goal}})$,
it follows that $\mu(\text{Vor}(\boldsymbol{v}))$ is bounded below by some positive number, say $s_{i_k}$. Now, note that for any arbitrary location of $\boldsymbol{v}$ in the set ${g}^{(i_k)}({\boldsymbol{X}_\text{goal}})$, there is a subset of $\text{Vor}(\boldsymbol{v})$, in which, if the new vertex is sampled, leads to extension of the 3D-OGSE-generated graph to at least one point in ${g}^{(i_k -1)}({\boldsymbol{X}_\text{goal}})$. Let the minimum size of such a subset over the set ${g}^{(i_k)}({\boldsymbol{X}_\text{goal}})$ be denoted as $\zeta_{i_k}$. We can then conclude that for $p=\underset{i_k \leq m}{\text{min}}{\frac{{s_{i_k} \zeta_{i_k}}}{{\mu(\mathbb{X}_{\text{free}})}}}$, the proposition holds true.
\end{proof}
Now, consider a Bernoulli distributed random variable $B_k$ with parameter $p$ as obtained in Proposition \ref{prop:parameter_p_berna}. 
Note that $B_k$ thus defined could be thought as a Bernoulli random variable used to model the \say{worst-case} probability of the 3D-OGSE-generated graph extending from a set ${g}^{(i_k)}({\boldsymbol{X}_\text{goal}})$ to a set ${g}^{(i_k -j)}({\boldsymbol{X}_\text{goal}})$ in the $k^\text{th}$ iteration for $j\leq i_k-1$.
Next, define random variable $B=B_1+B_2+\dots+B_N$. Following the arguments used in the proof for RRT \cite{rrt}, $\mathbb{P}(A<q)$ is obtained as upper-bounded as given in the following proposition.
\begin{prop}
$\mathbb{P}(A<q)< \mathbb{P}(B<q)$
\label{prop:proposition_a_lessthan_c}
\end{prop}
\begin{proof}
From Proposition \ref{prop:parameter_p_berna}, we know that $\mathbb{P}(A_i=1)\geq\mathbb{P}(B_i =1)=p$ for $i$ less than or equal to $N$. Considering the nature of variation of the probability distribution of Binomial distribution with variation in distribution parameter, it could be inferred that $\mathbb{P}(B<q)$ provides a worst-case (lower) bound on $\mathbb{P}(A>q)$, that is, $\mathbb{P}(A>q)\geq\mathbb{P}(B>q)$. Taking complement on both sides proves the proposition.
\end{proof}
Using a theorem for Chernoff bounds for sum of Bernoulli random variables $B_1,B_2\dots B_N$ \cite{chernoff},
\begin{align}
  \mathbb{P}(A<q)\leq \mathbb{P}(B<q)
  \leq e^{(-((N/p)+(({(qp)}^{2})/N )-2q))} 
 \label{eqn:upper_bound}
\end{align}
The upper bound of $\mathbb{P}(A<q)$ obtained above could be given in a more compact form as $\mathbb{P}(A<q)\leq ae^{-b N}$, where $a=e^{-((({(qp)}^{2})/N )-2q)}$ and $b=\frac{1}{p}$. 

Recall that the event $\{\mathcal{P}_N=0\}$ is a subset of the event $\{A<q\}$. In addition, from Eqn. \eqref{eqn:upper_bound} since  $\mathbb{P}(A<q)\leq ae^{-b N} $, therefore it can be concluded that $\mathbb{P}(\mathcal{P}_N=0)\leq ae^{-b N} $. Considering limit supremum operation on both sides gives the following, 
\begin{equation}
 \limsup_{N\to\infty}\;{\mathbb{P}(\mathcal{P}_N=0)}\leq \limsup_{N\to\infty}\;{\mathbb{P}(A<q)}=0
\end{equation}
Thus, the 3D-OGSE algorithm satisfies the probabilistic completeness criterion, i.e., there is a probabilistic guarantee that the 3D-OGSE generates a collision-free feasible path from $\boldsymbol{X}_{\text{init}}$ to $\boldsymbol{X}_{\text{goal}}$ over multiple planning iterations. 
\end{proof} 

The following corollary presents an upper bound of   the number of such planning iterations required to generate such a feasible path in expectation.

\begin{corollary}
\label{corollary:expectation}
\textit{Since the expectation of the random variable $B$ i.e. $\mathbb{E}(B)=N p$ (since it is a Binomial distribution), we can conclude that the expected number of iterations it takes for the 3D-OGSE algorithm to succeed is bounded above by $\frac{q}{p}$, for a space with restricted visibility index $q$.}
\end{corollary}

\subsection{Guarantee of Collision-free Trajectory by the 3D-OGSE Algorithm}

Theorem \ref{theorem:probabilistic_completeness} above gives the probabilistic guarantee of obtaining a collision-free feasible path generated over multiple planning iterations, in which all the vertices belong to $\mathbb{X}_{\text{free}}$. Now, let's consider the overall trajectory obtained over multiple planning iterations. The following proposition gives a guarantee that this trajectory is also collision-free.

\begin{prop}
\textbf{(Collision-Free Trajectories):} \textit{ The 3D-OGSE guarantees collision-free trajectories.}
\label{lemma:dynamic_feasible}
\end{prop}
\begin{proof}
Once a feasible path is generated at each planning phase, a smooth trajectory is computed by solving the snap minimization problem given in \eqref{optimization_problem}. Note that the last constraint in the optimization problem given in \eqref{eqn:velocity_and_acceleration_constraints} ensures that any point on the generated trajectory at that planning phase lies in free-space.

Now, let's consider the transition from one planning phase to the subsequent one. This re-planning is set by the $\texttt{Trigger}$ function module given in Line 9 of Algorithm \ref{alg:3d_online_gse} and in Algorithm \ref{alg:trigger_replanning_gse}. From these, it could be easily observed that the re-planning trigger decision is taken in one-step ahead manner, that is, if the next point (Line 1 of Algorithm \ref{alg:trigger_replanning_gse}) is beyond the sensing shape about the present point, then a re-planning is triggered even before the agent reaches the surface of the sensing shape, and subsequent planning phase begins. This ensures that the agent does not traverse to any obstacle-space even at the transition from one planning phase to another. This completes the proof. 
\end{proof}

Depending on the computational time required to generate a trajectory and also ensuring that the agent lies in the collision-free zone at every point, a multiple-step look ahead can be incorporated in Line 1 of the $\texttt{Trigger}$ algorithm (Algorithm \ref{alg:trigger_replanning_gse}). In other words instead of using the one-step look ahead $\sigma^k(j\Delta t)$, a multiple-step look ahead can be incorporated $\sigma^k((j+w)\Delta t)$ (in Line 1 of Algorithm \ref{alg:trigger_replanning_gse}), where $w$ steps are looked ahead.

{Besides this, also note that in several existing literature safe corridors are formed in a half-space generated by a hyperplane touching the closest point on an obstacle. In non-convex obstacle shape scenario, such half-spaces and safe corridors contained in it still might have intersection with obstacle-space making their approaches not suitable for non-convex obstacle shapes. To the contrary, the 3D-OGSE presented in this paper considers the notion of generalized shape for the formation of safe corridor, which is devoid of such problem of intersection with obstacle-space and hence, can be equally effectively leveraged for non-convex shape of obstacles also.}
\section{\label{sec:results}Results}
\subsection{Implementation}
Our method is implemented on an Intel Core i5-8500 CPU at 3.0 GHz with 32GB memory. We use ROS Melodic and Rviz for our simulation experiments. The trajectory optimization is performed using Mosek. 
%
We evaluate our method in simulation with prior state-of-the-art online trajectory generation methods presented by Liu et al.~\cite{upenn}, Gao et al.~\cite{hkust} and Usenko et. al~\cite{bsplineusenko2017real}. We present the benefits of the 3D-OGSE in terms of its low computation time, trajectory cost, and better performance in narrow passages and in high obstacle-dense environments.
\subsubsection{Computation Time and Trajectory Cost}
\begin{table}[t]
 \centering
 \scalebox{0.95}{
\begin{tabular}{|p{1.2cm}|p{0.85cm}|p{0.8cm}|p{1.0cm}|p{1.2cm}|p{1.2cm}|}
    \hline
     \multirow{2}{1cm}{Method} & & \multicolumn{3}{c|}{Computation Time (ms)} & \multirow{2}{1cm}{Trajectory Cost (snap (m/s$^4$)) } \\
     \cline {2-5}
     & & Total Time & Path Computation & Trajectory Generation & \\
     \hline
     \hline
    \multirow{3}{1.5cm}{3D-OGSE (FOV:$360\degree$)}&Mean&1.345&0.536&0.809&51.562\\
    \cline{2-6}
    &Max&4.745&0.934&3.811&96.562\\
    \cline{2-6}
    &Std dev&1.452&0.441&1.011&17.357\\
    \hline
     \multirow{3}{1.5cm}{3D-OGSE (FOV:$120\degree$)}&Mean&1.726&0.678&1.048&53.367\\
    \cline{2-6}
    &Max&4.923&1.127&3.796&102.378\\
    \cline{2-6}
    &Std dev&3.167&0.978&2.189&20.126\\
    \hline
    \multirow{3}{1.5cm}{Liu et al. \cite{upenn}}&Mean&9.324&3.371&5.952&60.577\\
    \cline{2-6}
    &Max&15.729&6.437&9.292&101.578\\
    \cline{2-6}
    &Std dev&2.245&0.378&1.873&20.178\\
    \hline
    \multirow{3}{1.5cm}{Gao and Shen \cite{hkust}} & Mean & 13.547 & 5.509 & 8.038 & 62.189\\
    \cline{2-6}
    &Max&25.265&10.638&14.627&105.272\\
    \cline{2-6}
    &Std dev&4.352&1.283&3.069&12.467\\
    \hline
    \multirow{3}{1.5cm}{Usenko et. al \cite{bsplineusenko2017real}} & Mean & 5.512 & - & 5.512 & 58.986\\
    \cline{2-6}
    &Max&7.653&-&7.653&96.463\\
    \cline{2-6}
    &Std dev&2.568&-&2.568&16.578\\
    \hline
\end{tabular}
 }
 \caption{ We tabulate the computation time and trajectory cost (snap) averaged over 100 generated trajectories for an agent operating in the random forest environment (Fig.~\ref{fig:real_traj_1}). We highlight the breakdown of the computation time between path generation and trajectory optimization for our algorithm 3D-OGSE and compare it with Liu et al. \cite{upenn}, Gao and Shen \cite{hkust}, and Usenko et. al \cite{bsplineusenko2017real}. Overall, we observe a $4-10$ times improvement in the performance over the prior methods, while our trajectory cost are comparable. {\color{black} Since \cite{bsplineusenko2017real} is a trajectory optimization based method, there is no path generation, thus the table entries for this case is empty.}}
\label{Tab:ComputationTime}
\vspace{0pt}
\end{table}

Table \ref{Tab:ComputationTime} summarizes the computation time and trajectory cost for 3D-OGSE and other algorithms under comparison.
We evaluate two implementations of 3D-OGSE, one with a sensor field-of-view (FOV) of $360\degree$, while the other one is for a restricted sensor FOV of $120\degree$. The values tabulated are averaged over 100 generated trajectories, and the methods were compared in the random forest environment (Fig.~\ref{fig:real_traj_1}) with the same start and goal positions. The maximum velocity and acceleration limits were taken as $3m/s$ and $2m/s^2$, respectively.
It can be observed 
that the trajectory cost of 3D-OGSE is comparable to prior methods under comparison, while the computation time for our method is significantly lower than them. The main advantage derived from the consideration of sensing shape based on the notion of generalized shape is that it leads to having fewer edges in the overall connected path. Subsequently, the number of connected segments in the QP formulation are few, thereby significantly reducing the computation time for optimizing trajectory about the generated path as is reflected in Table \ref{Tab:ComputationTime}. 

\subsubsection{Restricted FOV with Noisy Sensing}
To simulate realistic scenarios, we consider restricted FOV with noisy sensing by on-board sensor. While the FOV is limited to $120\degree$, the noisy sensing is characterized by  
the dimension of the sensing region $\boldsymbol{x}$ being chosen from a truncated Gaussian distribution with mean as its nominal sensing region dimension $(\mu_s)$ and standard deviation $(\sigma_s)$ as below.
\begin{align}
  f_s(\boldsymbol{x} , \mu_s, \sigma_s, a)=\frac{1}{\sigma_s} \frac{\Phi\left(\frac{\boldsymbol{x}-\mu_s}{\sigma_s}\right)}{\Phi\left(\frac{a}{\sigma_s}\right)-\Phi\left(\frac{-a}{\sigma_s}\right)}  
  \label{eqn:sensor_noise}
\end{align}
where, $\Phi$ is the cumulative Gaussian distribution. For simulations, we choose $\mu_s=5$m (half of side-length 10m of nominal cubic sensing region) and $\sigma_s=1m$. And, $a$ is selected as 0.1m, 0.3m and 0.5m indicating varied level of sensing uncertainty. From Table \ref{tab:sensor_noise_results} it can be observed that with increase in $a$ or sensing uncertainty, for low obstacle-density the number of collided trajectories increase linearly. However, for higher sensing uncertainty and higher obstacle densities the number of collided trajectories increase superlinearly. In this context, it is important to note that for ideal sensing with full and restricted FOV, our 3D-OGSE leads to no collided trajectory (result of which is given in Table \ref{Tab:ComputationTime}), which can also be verified from the theoretical guarantees given in Theorem \ref{theorem:probabilistic_completeness} and Proposition \ref{lemma:dynamic_feasible}.
\begin{table}[H]
    \centering
    \scalebox{0.99}{
    \begin{tabular}{ |c|c|c|c| } 
 \hline
 3D-OGSE& Low (1 obs./$m^2$) & Med.(2 obs./$m^2$) & High(3 obs./$m^2$) \\ \hline
 a=0.1m & 91 & 101 & 139 \\ \hline
 a=0.3m & 98 & 107 & 185 \\  \hline
 a=0.5m & 106 & 135 & 289 \\
 \hline
\end{tabular}}
    \caption{Table indicates the number of trajectories that collided with obstacles in a scenario with sensor noise \eqref{eqn:sensor_noise} and sensor FOV of $120\degree$. The number reported are out of a 1000 generated trajectories, in different environments with obstacle density varying from $1$ to $3 obs/m^2$, and $a$ between 2\% and 10\% of $\mu_s$.}
    \label{tab:sensor_noise_results}
\end{table}

\subsubsection{Obstacle Density}

Fig. \ref{fig:obstacle_density} shows the variation in computation time with an increase in the obstacle density in the random forest environment. For this evaluation, we varied the obstacle density from $1$ to $3$ obstacles$/\text{m}^2$. We observe that even with the increase in obstacle density, the corresponding increase in computation time is significantly lower in our method than the methods under comparison.
The computation times reported are averaged over 100 generated trajectories.

\begin{figure}[t]
\centering
\includegraphics[width=0.48\textwidth, height=3.5cm]{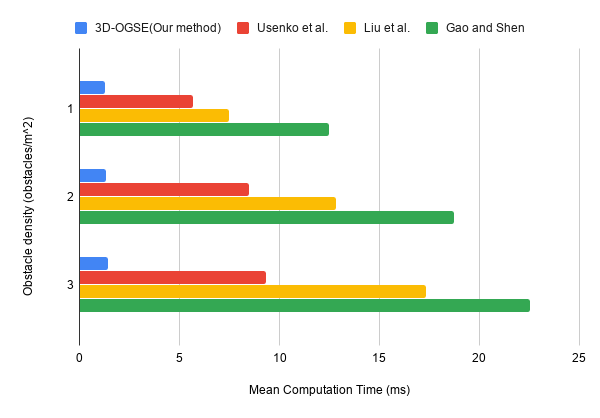}
\caption{We observe the trajectory computation time to increase with an increase in the obstacle density (obstacles per unit area) of the random forest environment. But, the observed increase in computation time for the 3D-OGSE is considerably smaller compared to prior methods by Liu et al. \cite{upenn}, Gao and Shen \cite{hkust}, and Usenko et al. \cite{bsplineusenko2017real}. The figure summarizes the computation times for the four methods for 1, 2, and 3 obstacles per Sq.m. The values stated are averaged over 100 trajectories generated for randomly chosen start and goal positions. }
\label{fig:obstacle_density}
\vspace{-5pt}
\end{figure}

\subsubsection{Narrow Passages}
\begin{figure}[t]
\centering
\centering
\begin{subfigure}{0.23\textwidth}
\includegraphics[width=4.3cm, height=3.2cm]{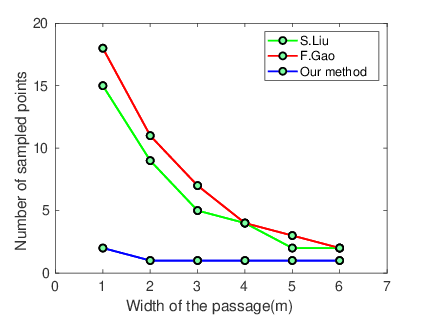}
\caption{}
\label{narrow_passages}
\end{subfigure}
\begin{subfigure}{0.23\textwidth}

\includegraphics[width=4.3cm, height=3.2cm]{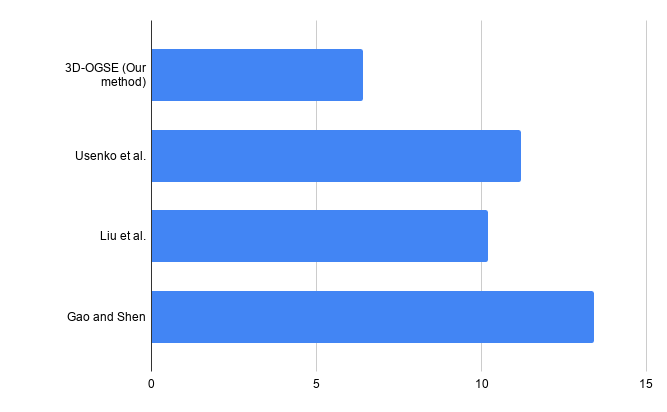}
\caption{}\label{replanning_triggers}
\end{subfigure}
\caption{(a):  We observe that our method samples lower number of points for the narrow passage than Liu et al. \cite{upenn} and Gao et al. \cite{hkust}. We evaluate on scenarios with passage width ranging from $1$ to $6m$.(b): The average number of replanning triggers averaged over 100 trajectories in random forest environment using our method, Liu et al. \cite{upenn}, Gao et al. \cite{hkust} and Usenko et al. \cite{bsplineusenko2017real}.}
\vspace{-12pt}
\end{figure}
We evaluate the performance of the algorithms in multiple narrow passages of fixed length and different widths. {{Fig. \ref{narrow_passages}}} shows the variation in the number of sampled points (in the narrow passage) when computing a feasible path and the passage width. We observe in our method the number of sampled points remain almost constant even with decrease in passage width, while the number of sampled points increases in~\cite{upenn,hkust}. 
Further from {\color{black}{Fig.~\ref{horse_shoe}}}, we observe our algorithm performs well even in complex scenarios such as the `horseshoe' environment.

\subsubsection{Replanning Triggers}
We compare the number of replanning triggers required to reach the goal position for our proposed method, \cite{upenn}, \cite{hkust}, and \cite{bsplineusenko2017real}. We considered the random forest environment for this scenario and generated 100 trajectories from a start to a goal position. The average number of replanning triggers required to reach the goal is then calculated for these 100 trajectories and is shown in Fig. \ref{replanning_triggers}. We observe that our methods required a lower number of replanning triggers on average than other methods under comparison.


\begin{figure}[t]
\centering
\begin{subfigure}{0.22\textwidth}
\centering
\centerline{\includegraphics[width=3.2cm,height=2.8cm]{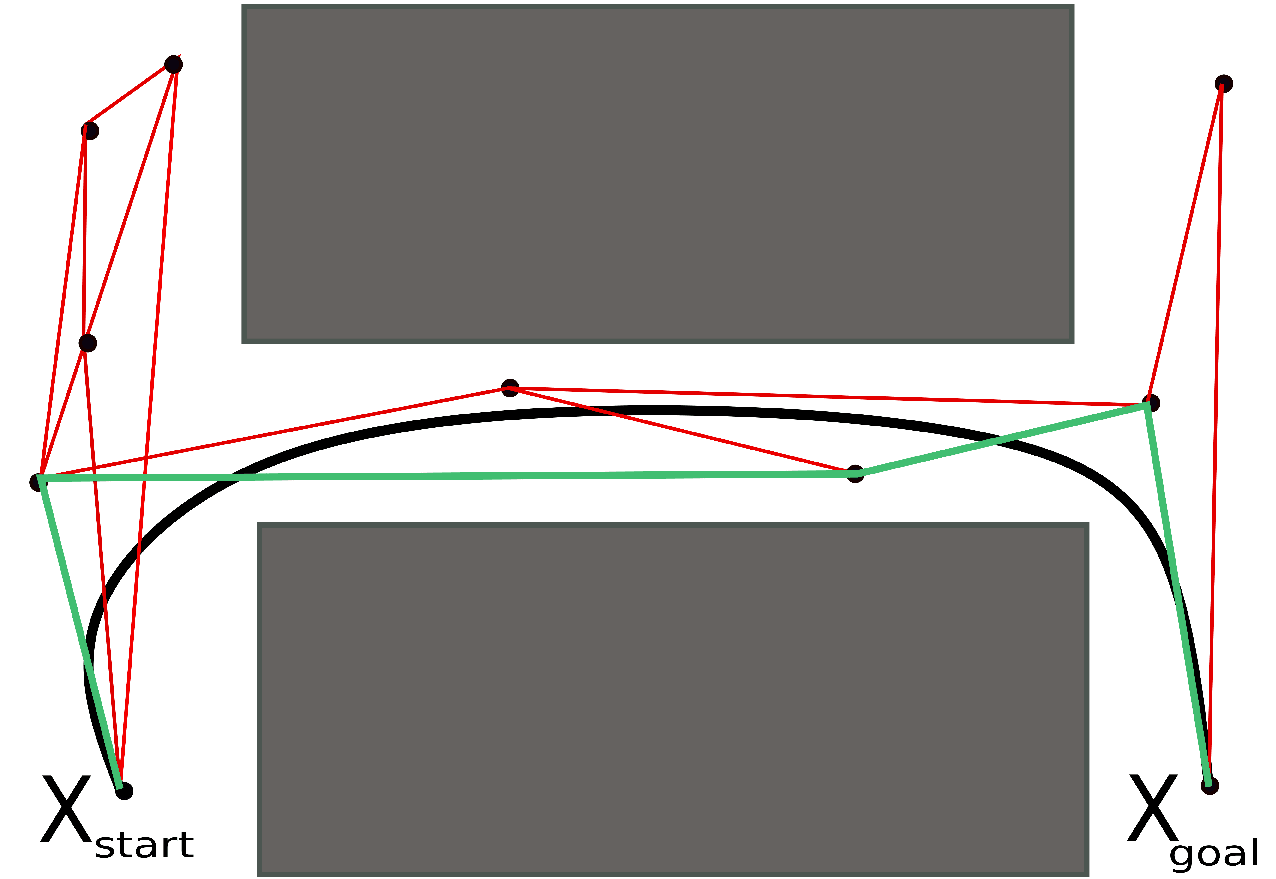}}
\caption{}
\label{fig:cases_for_shape}
\label{explain_shape}
\end{subfigure}
\begin{subfigure}{0.22\textwidth}
\centering

\centerline{\includegraphics[width=3.2cm,height=2.8cm]{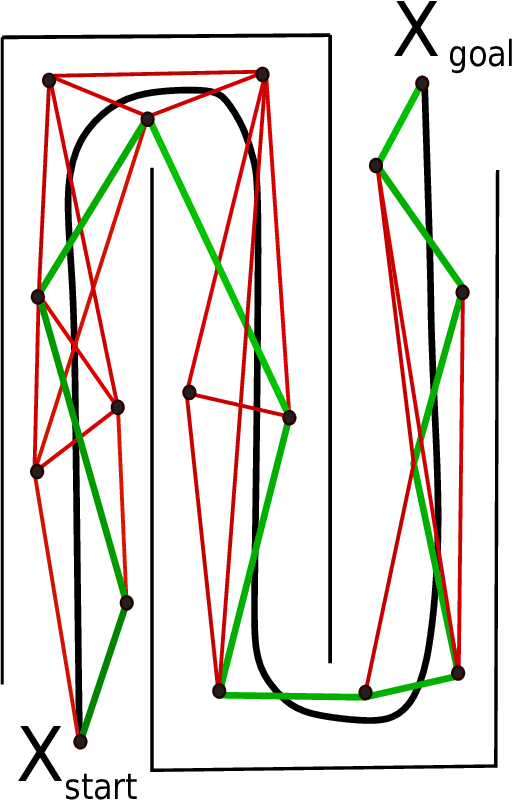}}
\caption{}
\label{convexify}
\end{subfigure}
\caption{A horseshoe environment and a thin corridor is considered. We observe that our proposed 3D-OGSE algorithm is able to generate a feasible trajectory from the start to goal location in these scenarios. The width of the passage in horseshoe environment and thin corridor is 1m.{The green lines denote the shortest path from $\boldsymbol{X}_\text{init}$ to $\boldsymbol{X}_\text{goal}$ and black lines denote the smooth trajectory between these points}}
\label{horse_shoe}
\end{figure}

\section{\label{sec:conclusion}Conclusion}
In this paper, we propose a novel online safe and smooth trajectory generation algorithm, 3D-OGSE, for autonomous navigation of a 3-D agent in an unknown 3-D obstacle-cluttered environment. Our algorithm uses local information about environment to continuously re-plan trajectories to direct the agent towards a goal position over multiple planning horizons / iterations that are mechanised by re-planning trigger of a receding horizon planner. Our algorithm is fast and requires $1.4$ ms on average to re-plan trajectories, which is approximately 4-10 times faster than \cite{upenn}, \cite{hkust}, and \cite{bsplineusenko2017real}. Also, in complex environments like narrow passages also the efficiency of the 3D-OGSE is found evident.
We plan on extending our algorithm to noisy sensor settings, where trajectories are to be generated by taking sensor noise also into account. Further, we plan on extending this method for scenarios with dynamic obstacles and evaluating the performance on test-bed with real-world 3-D agents .
\bibliography{root.bib}
\end{document}